\documentclass[lettersize,journal]{IEEEtran}     

\usepackage{makecell}
\usepackage{amsthm}
\usepackage{booktabs}
\usepackage{amssymb}
\usepackage{graphicx}
\usepackage{multirow}
\usepackage{color}
\usepackage{booktabs}
\usepackage{makecell}
\usepackage[figuresright]{rotating}
\usepackage{adjustbox}
\usepackage{footnote}
\usepackage{amsfonts}
\usepackage{lineno}
\usepackage{multirow}
\usepackage {amsmath}
\usepackage{graphicx}
\usepackage{latexsym}
\usepackage{authblk}
\usepackage{algorithmic}
\usepackage{algorithm}
\usepackage{array}
\usepackage[caption=false,font=normalsize,labelfont=sf,textfont=sf]{subfig}
\usepackage{textcomp}
\usepackage{stfloats}
\usepackage{url}
\usepackage{verbatim}
\usepackage{graphicx}
\usepackage{cite}
\begin{document}

\title{Few-shot Adaptation of Multi-modal \\ Foundation Models: A Survey}

\author{
Fan Liu$^{1}$,
Tianshu Zhang$^{1}$,
Wenwen Dai$^{1}$,
Wenwen Cai$^{1}$,
Xiaocong Zhou$^{1}$,
and Delong Chen$^{2}$

$^{1}$Hohai University \ \ 
$^{2}$HKUST

\texttt{delong.chen@connect.ust.hk}

\thanks{This work was partially supported by National Nature Science Foundation of China (62372155), Joint Fund of Ministry of Education for Equipment Pre-research (8091B022123), Research Fund from Science and Technology on Underwater Vehicle Technology Laboratory (2021JCJQ-SYSJJ-LB06905), Key Laboratory of Information System Requirements, No: LHZZ 2021-M04, Water Science and Technology Project of Jiangsu Province under grant No.2021063, Qinglan Project of Jiangsu Province.}
}

\maketitle

\begin{abstract}
    Multi-modal (vision-language) models, such as CLIP, are replacing traditional supervised pre-training models (e.g., ImageNet-based pre-training) as the new generation of visual foundation models. These models with robust and aligned semantic representations learned from billions of internet image-text pairs and can be applied to various downstream tasks in a zero-shot manner. However, in some fine-grained domains like medical imaging and remote sensing, the performance of multi-modal foundation models often leaves much to be desired. Consequently, many researchers have begun to explore few-shot adaptation methods for these models, gradually deriving three main technical approaches: 1) prompt-based methods, 2) adapter-based methods, and 3) external knowledge-based methods. Nevertheless, this rapidly developing field has produced numerous results without a comprehensive survey to systematically organize the research progress. Therefore, in this survey, we introduce and analyze the research advancements in few-shot adaptation methods for multi-modal models, summarizing commonly used datasets and experimental setups, and comparing the results of different methods. In addition, due to the lack of reliable theoretical support for existing methods, we derive the few-shot adaptation generalization error bound for multi-modal models. The theorem reveals that the generalization error of multi-modal foundation models is constrained by three factors: domain gap, model capacity, and sample size. Based on this, we propose three possible solutions from the following aspects: 1) adaptive domain generalization, 2) adaptive model selection, and 3) adaptive knowledge utilization. 
\end{abstract}

\begin{IEEEkeywords}
Multi-modal foundation models, Vision-language pre-training, Few-shot learning, Parameter efficient fine-tuning.
\end{IEEEkeywords}

\section{Introduction}
\label{intro}
\IEEEPARstart{A}{rtificial} intelligence is increasingly being applied to a wide range of key industries, including voice recognition, image recognition, autonomous driving, intelligent manufacturing, medical diagnosis, financial risk control and so on. In the process of empowering various fields with artificial intelligence technology, there are often challenges related to fragmented and diversified demands. In the past, models often had small parameter sizes and limited generalization capabilities. And one model could only cope with a single scenario, resulting in high costs and poor generalization performance. Recently, an increasing number of researchers have started focusing on pre-trained foundation models with greater generalization.

Since 2018, the training data and parameter sizes of foundation models such as BERT \cite{Devlin2019BERT}, Pangu \cite{Zeng2021PanGu}, PaLM \cite{Chowdhery2022PaLM}, GPT-4 \cite{OpenAI2023GPT}, etc. have grown exponentially, resulting in significant performance improvements in various natural language understanding tasks. Meanwhile, the development of foundation models is gradually evolving from single modalities such as text, speech, vision, etc. to multi-modal fusion. More and more research organizations have turned their attention to multi-modal pre-trained foundation models, such as ViLBERT \cite{Lu2019ViLBERT}, CLIP \cite{Radford2021Learning}, DeCLIP \cite{Li2022Supervision}, FILIP \cite{Yao2022FILIP}, PyramidCLIP \cite{Gao2022PyramidCLIP}, OFA \cite{Cai2020Once}, BEiT-3 \cite{Wang2022Image}, ERNIE-ViL \cite{Shan2022ERNIE} and Data2vec \cite{Baevski2022data2vec}.

In early 2021, OpenAI released CLIP, a large-scale multi-modal model for aligning images and texts, which is pre-trained using billions of internet data to obtain rich visual language knowledge through contrastive learning. While the pre-trained CLIP model can achieve zero-shot predictions by employing text features as classification weights during the inference stage, this approach typically excels only in general domains like ImageNet and tends to underperform when dealing with data from certain fine-grained domains. The reason behind this is that such models primarily utilize data from the general domain during their pre-training phase, and when confronted with specific downstream tasks, the data distribution often diverges from pre-training data. Hence, it becomes necessary to fine-tune the model using the specific data of downstream tasks. To improve the generalization performance of the model through fine-tuning, researchers first proposed a prompt-based fine-tuning adaptation method (e.g., CoOp \cite{Zhou2022Learning}), which treats the fixed text inputs of the CLIP text side as learnable vectors and then fine-tunes them with a small number of samples to adapt to the downstream tasks. Another method commonly employed to enhance the few-shot adaptation capability is adapter-based fine-tuning, like CLIP-Adapter \cite{Gao2021CLIP}. This method involves adding simple adapter structures within the pre-trained model and then fine-tuning the adapter parameters using a small amount of sample data, enabling the foundation model to adapt to downstream tasks. In addition, methods that introduce foundation language models or external knowledge such as knowledge graphs (e.g., CuPL \cite{Pratt2022What}) can help the model to handle unseen samples better, enhance its semantic comprehension and robustness, and thus improve its performance in few-shot adaptation tasks. The above-mentioned three kinds of methods have been widely used in various downstream adaptation tasks, but there is a lack of a comprehensive survey that systematically sorts out the methods. Therefore, we elaborate and compare these methods in detail and explore their future directions to further improve the performance and generalization ability of pre-trained models. Contributions of this paper are as follows:

 \begin{itemize}
     \item We comprehensively review and sort out multi-modal few-shot adaptation methods, and classify existing methods into prompt-based fine-tuning adaptation methods, adapter-based fine-tuning adaptation methods, adaptation methods based on external knowledge, and other methods. Within the prompt-based fine-tuning adaptation methods, we further subdivide them into text prompt fine-tuning, visual prompt fine-tuning, multi-modal prompt, and multi-task prompt methods. Regarding adapter-based fine-tuning adaptation methods, we categorize them into single-modal adapter fine-tuning and multi-modal adapter fine-tuning. As for methods employing external knowledge, we distinguish between pre-training methods with external knowledge and downstream adaptation methods leveraging external knowledge.
     \item We review 11 commonly used datasets for evaluating the downstream generalization performance of multi-modal foundation models. We provide a detailed description of four experimental setups for verifying the adaptation performance of multi-modal foundation models under few-shot conditions. The experimental results for the four different setups are presented, and a comparative analysis of these results is performed. We highlight the reasons why different types of methods can effectively enhance the generalization performance of multi-modal foundation models.
     \item We discuss the common shortcomings of few-shot adaptation methods for existing multi-modal foundation models and analyze the domain adaptation problem. Starting from the error bound in cross-domain generalization from statistical machine learning theory, we derive the error bound for few-shot adaptation with multi-modal foundation models, which reveals that the main challenges faced by existing methods are ineffective adaptation of upstream and downstream domain distributions, lack of adaptability in model selection and insufficient utilization of data and knowledge.
 \end{itemize}

\section{Pre-training of Multi-modal Foundation Models }
\label{sec: Pre-training of Multi-modal Foundation Models}
In recent years, large-scale pre-training models have received extensive attention from academia and industry. Initially, the related works of foundation model pre-training mainly focus on the field of natural language processing, in which the self-supervised language models such as BERT \cite{Devlin2019BERT} and GPT \cite{Radford2018Improving} have shown better natural language understanding and generation capabilities than traditional methods. In the field of computer vision, the paradigm has also shifted from supervised pre-training to self-supervised pre-training. The performances of self-supervised pre-trained visual models have significantly improved, evolving from initial models based on data augmentation like SimCLR \cite{Chen2020Simple} and MoCo \cite{He2020Momentum} to more recent approaches based on random masking methods such as MAE \cite{He2022Masked} and BEiT \cite{Bao2022BEiT}. However, pre-trained language models are unable to receive visual inputs, resulting in an inability to extend their advantage in language understanding to multi-modal downstream tasks such as visual question answering (VQA). On the other hand, the supervised signals used for visual pre-training are often limited to data augmentation and stochastic masks, which prevents them from learning richer semantic representations in the open world. As a result, we have witnessed a recent surge in the development of large-scale pre-trained multi-modal models that combine visual and language modalities, as illustrated in Table \ref{tab:TABLE1}.

\begin{table}[htb]\centering
\caption{Multi-modal pre-trained foundation models}
\label{tab:TABLE1}
\setlength{\tabcolsep}{1mm}{
\renewcommand{\arraystretch}{1.5}
\begin{tabular}{cccc}
\hline
\textbf{Date} & \textbf{Institution} & \textbf{Publication}  & \textbf{Method}   \\ \hline
2021.02 & OpenAI     & ICML 2021    & CLIP \cite{Radford2021Learning}         \\ 
2021.06 & Google     & ICML 2021    & ALIGN \cite{Jia2021Scaling}            \\ 
2021.07 & Renmin University of China       & ArXiv        & BriVL \cite{Huo2021WenLan}        \\
2021.09 & Kuaishou Technology         & ArXiv        & EffificientCLIP \cite{Wang2021EfficientCLIP} \\
2021.09 & SenseTime Research         & ICLR 2022    & DeCLIP \cite{Li2022Supervision}            \\
2021.11 & HUAWEI     & ICLR 2022    & FILIP \cite{Yao2022FILIP}            \\
2021.12 & Facebook   & ArXiv        & BoW \cite{Tejankar2021Fistful}    \\
2021.12 & Facebook   & ECCV 2022        & SLIP \cite{Mu2022SLIP}            \\ 
2022.02 & Salesforce & ICML 2022    & BLIP \cite{Li2022BLIP}             \\
2022.02 & JKU        & NeurIPS 2022 & CLOOB \cite{Fuerst2022CLOOB}          \\ 
2022.06 & MEGVII         & IEEE TNNLS        & ProtoCLIP \cite{Chen2022Prototypical}       \\
2022.09 & IDEA       & Arxiv        & Taiyi-CLIP \cite{Wang2022Fengshenbang}      \\
2022.11 & DAMO Academy        & ArXiv        & ChineseCLIP \cite{Yang2022Chinese}     \\ 
2023.03 & BAAI        & ArXiv        & EVA-CLIP \cite{Sun2023EVA}      \\ \hline
\end{tabular}}
\end{table}

A notable characteristic of the above multi-modal pre-trained foundation models lies in the ability to efficiently learn visual concepts from large-scale natural language supervision and embed image and text features into a shared semantic space, thus obtaining zero-shot prediction capability. However, when the downstream task's data belongs to some specific domains, such as remote sensing, healthcare, e-commerce, etc., which differ greatly from the pre-training data, the zero-shot prediction accuracy of the multi-modal foundation models will drop sharply. At this point, it is necessary to fine-tune the model with the help of the downstream task's data, for example, using linear probing or global fine-tuning methods. However, such methods often require a large number of samples for effective training, and the number of samples available in the actual downstream task is often limited by the tagging cost. To address this problem, there have been some initial explorations in the academic community that attempt to fine-tune multi-modal foundation models using small amounts of data so that they can be efficiently generalized to specific downstream applications. For example, there have been some works \cite{Lin2023Multimodality,Gao2021CLIP} to fine-tune CLIP, such as using linear classifiers, adapter layers, etc. The work on fine-tuning CLIP can achieve very good results on few-shot image recognition tasks, even surpassing some algorithms designed specifically for few-shot tasks.

\section{Few-shot Adaptation Methods for Multi-modal Foundation Models}
To effectively enhance the model's generalization performance in specific domains, it is necessary to fine-tune multi-modal foundation models using limited samples, enabling them to have broader applications. These methods can be defined as few-shot adaptation methods for multi-modal foundation models. This chapter will be divided into four sections to provide a detailed overview of existing methods for multi-modal foundation models, namely: prompt-based fine-tuning adaptation methods, adapter-based fine-tuning adaptation methods, adaptation methods based on external knowledge, and other methods.

\subsection{Prompt-based Fine-tuning Adaptation Methods}
\subsubsection{Textual Prompt-based Fine-tuning Adaptation}

In the field of natural language processing, prompt-based fine-tuning adaptation \cite{Lester2021Power,Shin2020AutoPrompt,Li2021Prefix,Reynolds2021Prompt,Liu2021GPT} is a classic approach to addressing the issue of few-shot generalization in large language models. It involves using a fixed part of the text input as a learnable vector and fine-tuning its parameters using downstream task data, enabling the model to adapt to specific downstream tasks. The advantage of this method lies in its ability to avoid the manual design of textual prompts, effectively mitigating the risk of overfitting by fine-tuning only a specific portion of the model input. Inspired by this, some researchers have also begun to design prompt-based fine-tuning adaptation methods for multi-modal foundation models. CoOp \cite{Zhou2022Learning} for the first time incorporates the idea of prompt learning into downstream task adaptation for multi-modal pre-trained foundation models. It uses learnable word embeddings to automatically construct context prompts instead of manually designing prompt templates for each task. As illustrated in Figure \ref{fig:fig1}, the individual category label $\left\{ object\right\} $ is transformed into a comprehensive textual prompt '$\left[ V\right] _{1},\left[ V\right] _{2},\ldots ,\left[ V\right] _{m},\left\{ object\right\} $'. Here, $\left[ V\right] _{i}$ represents the adaptable word vectors. The classification loss is then computed to fine-tune these word vectors using data from the downstream task, enabling the model to autonomously acquire text inputs adapted to the downstream task.

\begin{figure}
    \centering
    \includegraphics[width=0.5\textwidth]{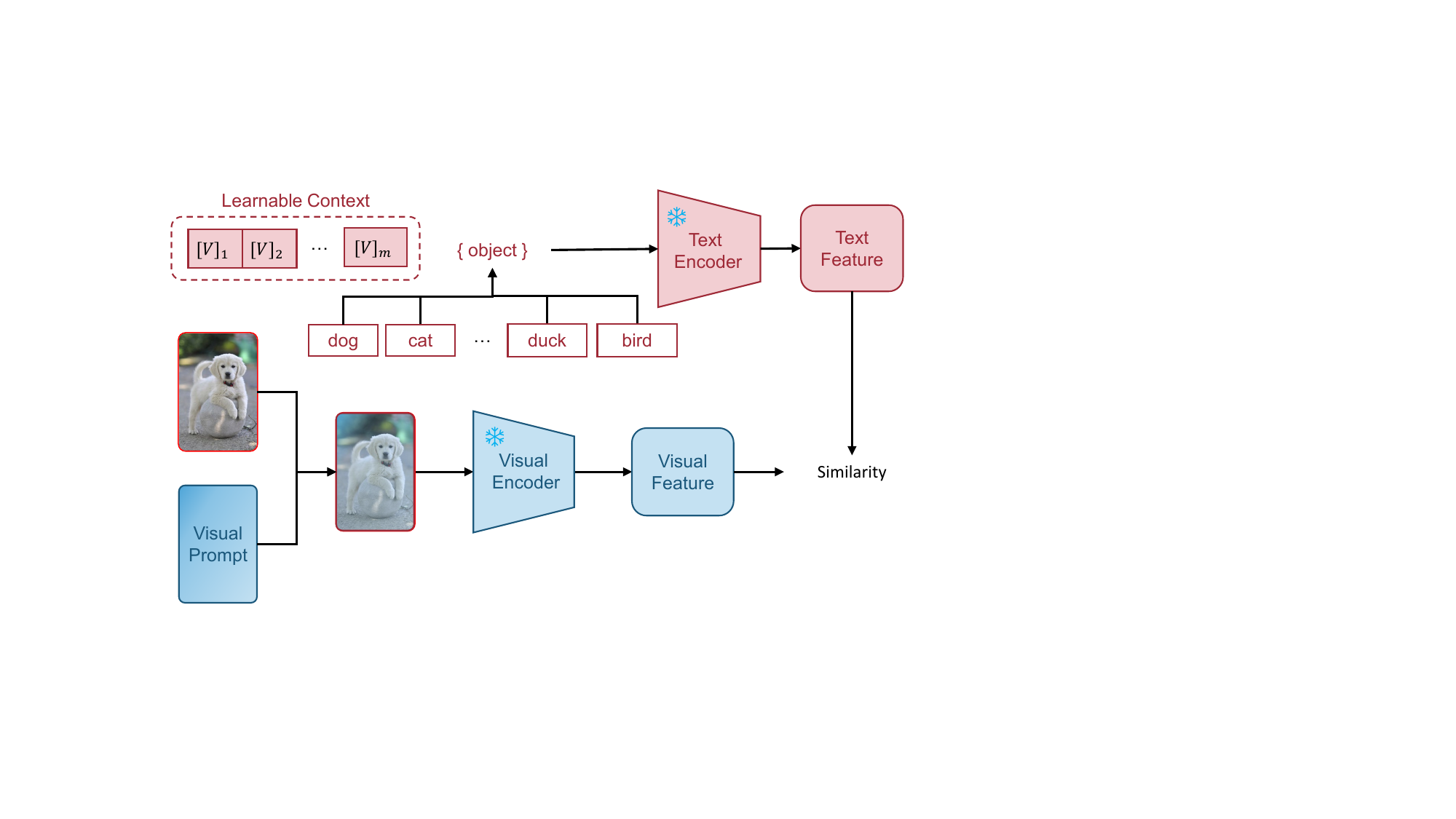}
    \caption{Schematic diagram of prompt-based fine-tuning adaptation methods.}
    \label{fig:fig1}
\end{figure}
  
Subsequently, Zhou et al. \cite{Zhou2022Conditional} introduced Conditional Contextual Optimization (CoCoOp), which constructs a meta-network to learn features from images. These features are then combined with prompt vectors to enhance CoOp's generalization performance on new category data. To leverage the zero-shot ability of pre-trained models effectively, Huang et al. \cite{Huang2022Unsupervised} proposed Unsupervised Prompt Learning (UPL). It selects zero-shot prediction results with high confidence as pseudo-labels to supervise prompt vector learning. Similarly, Prompt-aligned Gradient (ProGrad) \cite{Zhu2022Prompt} uses the zero-shot prediction results to constrain the direction of the model gradient update, thus avoiding the conflict between few-shot models and the generalized knowledge, and mitigating the problem of overfitting. However, due to the rich diversity of visual information, learning only one textual prompt makes it challenging to match complex visual data. To address this, Chen et al. \cite{Chen2023PLOT} proposed Prompt Learning with Optimal Transport (PLOT). It is used to learn multiple distinct textual prompts, where different textual prompts are regarded as descriptions of image locations, and the optimal transport theory is employed to match textual prompts with local image features. Lu et al. \cite{Lu2022Prompt} introduced Prompt Distribution Learning (ProDA) to learn prompt distributions and sample different textual prompts from these distributions. In addition, to make full use of the correlation between multi-task data, Ding et al. \cite{Ding2022Prompt} proposed Soft Context Sharing for Prompt Tuning (SoftCPT), which designs a task-sharing meta-network that splices predefined task names and learnable meta-prompts as inputs to fine-tune the prompts with the help of multi-task data.

\subsubsection{Visual Prompt-based Fine-tuning Adaptation}

All of the above methods only fine-tune the textual side of CLIP, whereas CLIP, as a multi-modal model, places equal importance on both visual and textual sides. Fine-tuning only the textual prompts cannot improve the ability of the visual encoder to extract features, and the extracted visual features are likely to mismatch the target features of downstream tasks. Therefore, inspired by the textual prompt fine-tuning adaptation, a series of visual prompt fine-tuning adaptation methods have emerged. Existing visual prompt fine-tuning adaptation methods mainly include token-level fine-tuning adaptation and pixel-level fine-tuning adaptation. Visual Prompt Tuning (VPT) \cite{Jia2022Visual} introduces learnable visual prompts in token form. Class-Aware Visual Prompt Tuning (CAVPT) \cite{Xing2022Class} further includes a cross-attention module on this basis to make visual prompts more focused on the objectives of downstream tasks. In contrast to token-based methods, Bahng et al. \cite{Bahng2022Exploring} suggested adding pixel-level visual prompts directly around the image in a padding format to enhance visual prompts. Wu et al. \cite{Wu2022Unleashing} further proposed Enhanced Visual Prompting (EVP) by scaling and padding instead of padding directly around the original image.

\subsubsection{Multi-modal Prompt-based Fine-tuning Adaptation}

In addition to separately learning textual and visual prompts, it is also possible to simultaneously learn multi-modal prompts to better align textual and visual features. Textual and visual features have inherent differences, and in order to strengthen the connection between them when learning multi-modal prompts, Multi-modal Prompt Learning (MAPLE) \cite{Khattak2022MaPLe} uses copula functions to transform textual prompts into visual prompts. Unified Prompt Tuning (UPT) \cite{Zang2022Unified} on the other hand, first learns a universal prompt and then decomposes it into textual and visual prompts. On the other hand, Multi-task Visual Language Prompt Tuning (MVLPT) \cite{Shen2022Multitask} introduces the concept of multi-task learning, fine-tuning textual and visual prompts using cross-task knowledge.

\subsection{Adapter-based Fine-tuning Adaptation Methods}
\subsubsection{Single-modal Adapter-based Fine-tuning Adaptation}
In the field of Natural Language Processing (NLP), the concept of Adapters was first introduced by the Google team in 2019 for fine-tuning large language models \cite{Houlsby2019Parameter}. During training on downstream tasks, this method freezes the parameters of the original language model and only updates a small number of parameters added as adapter modules. Due to its advantages such as parameter efficiency, flexibility in design, and high robustness, this approach has received extensive research attention in the NLP field in recent years \cite{Ding2022Delta}. More recently, the adapter-based approach has also been applied to Vision Transformers (ViTs) in the computer vision domain. Jie et al. \cite{Jie2022Convolutional} addressed the issue of the lack of inductive bias of adapter structures in ViTs by introducing Convolutional Bypasses (Convpass). Additionally, they proposed Factor-Tuning (FacT, cited as \cite{Jie2023FacT}) to further improve the efficiency of parameter-efficient transfer learning to meet storage constraints in practical applications.

\subsubsection{Multi-modal Adapter-based Fine-tuning Adaptation}

The above adapter-based methods are all applicable to single-modal foundation models in natural language processing or computer vision. In recent years, adapter-based methods have also been extended to multi-modal foundation models to enhance downstream generalization ability. Gao et al. \cite{Gao2021CLIP} introduced CLIP-Adapter, which adds a fully connected layer adapter after freezing the backbone network to learn additional knowledge. It then merges this knowledge with zero-shot prediction results based on residual connections, as illustrated in Figure \ref{fig:fig2}.

\begin{figure}
    \centering\includegraphics[width=0.5\textwidth]{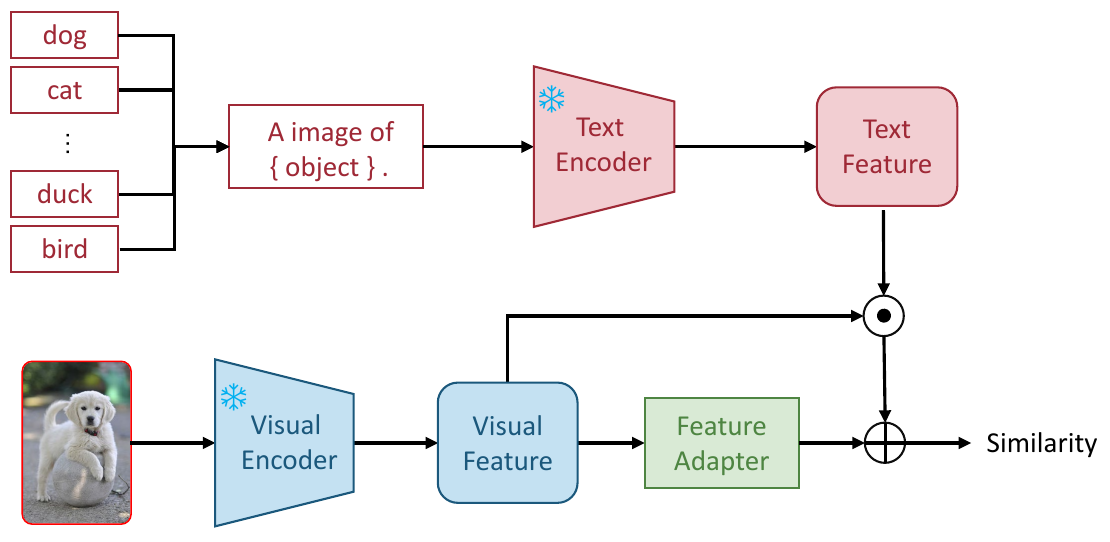}
    \caption{Schematic diagram of adapter-based fine-tuning adaptation methods.}
    \label{fig:fig2}
    \end{figure}

Building upon these developments, Zhang et al. introduced Tip-Adapter \cite{Zhang2021Tip}. This method constructs classifiers based on downstream few-shot training data and combines their predictions with the original zero-shot classifiers' results in a linear weighted manner to enhance the model's prediction performance. And SVL-Adapter \cite{Pantazis2022SVL} fuses a pre-trained self-supervised visual encoder before the adapter to extract more robust visual features. However, the above methods only use cross-modal contrastive loss and do not consider visually specific contrastive loss for few-shot datasets. To address this issue, Peng et al. \cite{Peng2022SgVA} proposed Semantic-guided Visual Adapting (SgVA-CLIP), which guides the parameter update of the visual adapter through implicit knowledge distillation to ensure the consistency of the image-text relationship. To enhance the cross-modal interaction capabilities of adapters, CALIP \cite{Guo2023CALIP} leverages attention maps to fuse text and image features and inserts two fine-tunable linear layers before and after fusion. In addition, Cross-Modal Adapter (CMA) \cite{Jiang2022Cross} and Multimodal Video Adapter (MV-Adapter) \cite{Zhang2023Multimodal} achieve cross-modal interaction by sharing adapter weights between two modalities. These methods consider both single-modal and multi-modal scenarios but do not fully integrate the advantages of each modality. To address this, Lu et al. \cite{Lu2023UniAdapter} proposed UniAdapter to unify single-modal and multi-modal adapters.

\subsection{External Knowledge-based Adaptation Methods}
\subsubsection{External Knowledge-based Pre-training Methods}

Pre-trained foundation models have the ability to learn general representations by mining relevant information from vast amounts of data on the internet. However, in such data-driven models, knowledge is often implicit and not explicitly linked to human understanding of the world or common sense knowledge. In recent years, data and knowledge-driven pre-training methods have been emerging, and researchers have started exploring the incorporation of more comprehensive external knowledge, such as knowledge graphs, into foundation models. This integration aims to make these models more robust, reliable, and interpretable. ERNIE \cite{Zhang2019ERNIE} incorporates a knowledge encoder for entity knowledge extraction and heterogeneous information fusion. K-BERT \cite{Liu2020K} retrieves external knowledge relevant to the model input and constructs sentence trees with rich contextual knowledge as model input. In recent years, some efforts have also started to inject knowledge into pre-training for multi-modal foundation models. For example, ERNIE-ViL \cite{Yu2021ERNIE} integrates knowledge from scene graphs, KM-BART \cite{Xing2021KM} models general visual knowledge by creating additional pre-training tasks, and K-LITE \cite{Shen2022K} incorporates various external knowledge sources, including WordNet and Wikipedia definitions.

\subsubsection{External Knowledge-based Downstream Adaptation Methods}
\begin{figure}
    \centering\includegraphics[width=0.5\textwidth]{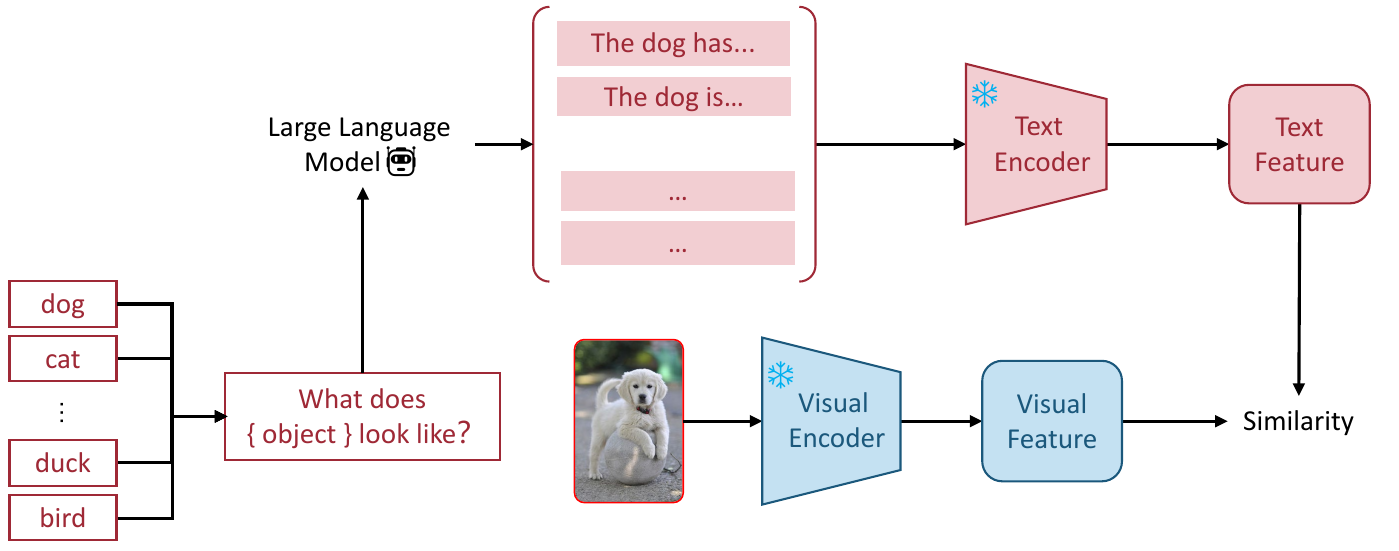}
    \caption{Schematic diagram of external knowledge-based fine-tuning adaptation methods.}
    \label{fig:fig3}
    \end{figure}  
    
The methods mentioned above introduce external knowledge during the pre-training phase. However, in downstream few-shot adaptation scenarios with limited data samples, it is also necessary to enhance external knowledge to ensure the model's performance. One of the most common approaches is to generate richer textual descriptions for each category by querying a large language model. An illustration of this method is shown in Figure \ref{fig:fig3}. Customized Prompts via Language models (CuPL) \cite{Pratt2022What} is the first method to integrate external knowledge into the downstream generalization process of multi-modal foundation models. CuPL achieves this by asking GPT-3 questions to generate multiple descriptive statements for each category, enriching the semantics of categories and thereby improving zero-shot classification performance. However, the sentences generated by CuPL using GPT-3 may have issues with poor descriptiveness and reliability. To address these issues, Menon et al. \cite{Menon2023Visual} further refined the knowledge enhancement process based on GPT-3. They prompted GPT-3 to generate semantic attribute descriptions in the form of phrases, enhancing the model's interpretability. To strike a balance between interpretability and performance, Language Guided Bottlenecks (LaBo) \cite{Yang2022Language} uses GPT-3 to generate a large candidate feature descriptor space, taking into account both the discriminability of features with respect to other classes and the coverage of the current class. It filters out the optimal sub-descriptor space for classification decisions, thereby uncovering the model's decision rationale. ELEVATER \cite{Li2022ELEVATER} also incorporates definitions from sources like GPT-3, WordNet, and Wiktionary. Experimental results indicate that external knowledge can enhance downstream generalization performance for multi-modal foundation models. However, different sources of knowledge have different emphases and properties. For instance, WordNet has relatively rich and accurate knowledge but lower coverage, while GPT-3 has a broader knowledge coverage but may lack reliability. Additionally, unlike the methods mentioned above that use external knowledge to enhance textual semantics, SuS-X \cite{Udandarao2022SuS} focuses on enhancing visual samples for multi-modal models. It augments the few-shot training sets for downstream tasks through image retrieval from the LAION-5B dataset \cite{Schuhmann2022LAION} or generates image samples based on Stable Diffusion \cite{Rombach2022High}, which aims to model the true distribution of downstream data more accurately and reliably.

\subsection{Other Methods}

In addition to the three categories of methods mentioned above, there are some approaches to fine-tuning multi-modal foundation models from the perspectives of weight parameter fusion, model reconstruction, and cross-attention. Specifically, Wise-FT \cite{Wortsman2022Robust} fuses the original and fine-tuned model parameters by linear interpolation, which enables the model to acquire specific knowledge from downstream data while retaining as much generic knowledge as possible. MaskCLIP \cite{Zhou2022Extract} directly modifies the structure of the CLIP image encoder by removing the query embedding layer and the key embedding layer. It replaces the value embedding layer and the last linear layer with a 1×1 convolutional layer, allowing the model to extract denser image features. VT-Clip \cite{Zhang2021VT} introduces a visual-guided attention mechanism, which enhances the semantic correlation between the textual features and the image data of the downstream tasks, thus effectively improving the generalization performance of the multi-modal foundation models.

\section{Datasets and Comparison of Experimental Results}

There are 11 commonly used datasets for evaluating the downstream generalization performance of multi-modal foundation models, namely: 2 general target datasets (ImageNet \cite{Deng2009ImageNet} and Caltech101 \cite{FeiFei2007Learning}), 5 fine-grained classification datasets (OxfordPets \cite{Parkhi2012Cats}, StanfordCars \cite{Krause20133D}, Flowers102 \cite{Nilsback2008Automated}, Food101 \cite{Bossard2014Food} and FGVCAircraft \cite{Maji2013Fine}), 1 scene recognition dataset (SUN397 \cite{Xiao2010SUN}), 1 action recognition dataset (UCF101 \cite{Soomro2012UCF101}), 1 texture dataset (DTD \cite{Cimpoi2014Describing}) and 1 satellite image dataset (EuroSAT \cite{Helber2019EuroSAT}). These datasets cover a range of different visual tasks and collectively form a more comprehensive benchmark for evaluating multi-modal foundation model performance in various scenarios.

To evaluate the generalization performance of multi-modal foundation models under few-shot conditions, four experimental settings are commonly used, namely:

\textbf{Few-shot Learning:} Building upon the 11 datasets mentioned above, the training and test sets are partitioned. For each class in the training set, 1, 2, 4, 8, and 16 samples are extracted and used for training. Subsequently, the model's performance is evaluated on the test set. The primary aim of this experiment is to assess the impact of limited samples on the generalization performance.

\textbf{Base-to-new Generalization:} To evaluate the effectiveness of adaptation methods for multi-modal foundation models on previously unseen classes, all the classes from the 11 datasets are evenly divided into two groups. One group is called ‘base classes’ and the other group is called ‘new classes’. The multi-modal foundation model is trained only on the data from the base classes. Subsequently, evaluations are performed separately on both the base class and new class data. Performance on the base classes reflects the discriminability of features learned by the model, while performance on the new classes reflects the model's generalization ability. The harmonic mean of the results obtained on the base and new class data is adopted as a balance between discriminability and generalization ability.

\textbf{Domain Generalization:} To validate the generalization and domain shift capabilities of multi-modal foundation models' adaptation methods when dealing with Out-of-Distribution (OOD) data, ImageNet is selected as the source dataset, and the other four datasets (ImageNetV2 \cite{Recht2019Do}, ImageNet-Sketch \cite{Wang2019Learning}, ImageNet-A \cite{Hendrycks2021Natural} and ImageNet-R \cite{Hendrycks2021Many}) are selected as the target datasets. The target datasets have the same category information as the source dataset but different data distributions. The model is trained solely on the source dataset and subsequently evaluated on the target datasets.

\textbf{Cross-Dataset Transfer:} To validate the generalization performance on different datasets, ImageNet is chosen as the source dataset, and the remaining 10 datasets are selected as target datasets. The model is trained on ImageNet and then tested on the target datasets. The source dataset and the target datasets have almost no overlap in classes, which can test the model's generalization ability on different class datasets. 

\begin{table*}[htb]
\caption{Comparison of few-shot learning experimental results \cite{Zhang2023Vision}}
\centering
\label{tab:TABLE2}
\resizebox{\textwidth}{!}{
\renewcommand{\arraystretch}{1.5}
\begin{tabular}{cccccccccccccc}
\hline
\textbf{Method}     & \textbf{Image Encoder} & \rotatebox{90}{\begin{tabular}[c]{@{}l@{}}\textbf{Average}\end{tabular}} & \rotatebox{90}{\begin{tabular}[c]{@{}l@{}}\textbf{ImageNet \cite{Deng2009ImageNet}}\end{tabular}} & \rotatebox{90}{\begin{tabular}[c]{@{}l@{}}\textbf{Caltech101 \cite{FeiFei2007Learning}}\end{tabular}} & \rotatebox{90}{\begin{tabular}[c]{@{}l@{}}\textbf{Pets \cite{Parkhi2012Cats}}\end{tabular}} & \rotatebox{90}{\begin{tabular}[c]{@{}l@{}}\textbf{Cars \cite{Krause20133D}}\end{tabular}} & \rotatebox{90}{\begin{tabular}[c]{@{}l@{}}\textbf{Flowers102 \cite{Nilsback2008Automated}}\end{tabular}} & \rotatebox{90}{\begin{tabular}[c]{@{}l@{}}\textbf{Food101 \cite{Bossard2014Food}}\end{tabular}} & \rotatebox{90}{\begin{tabular}[c]{@{}l@{}}\textbf{Aircraft \cite{Maji2013Fine}}\end{tabular}} & \rotatebox{90}{\begin{tabular}[c]{@{}l@{}}\textbf{SUN397 \cite{Xiao2010SUN}}\end{tabular}} & \rotatebox{90}{\begin{tabular}[c]{@{}l@{}}\textbf{DTD \cite{Cimpoi2014Describing}}\end{tabular}}  & \rotatebox{90}{\begin{tabular}[c]{@{}l@{}}\textbf{EuroSAT \cite{Helber2019EuroSAT}}\end{tabular}} & \rotatebox{90}{\begin{tabular}[c]{@{}l@{}}\textbf{UCF101 \cite{Soomro2012UCF101}}\end{tabular}} \\ \hline
Baseline \cite{Radford2021Learning}    & ViT-B/16      & 71.7    & 70.2     & 95.4       & 94.1 & 68.6 & 74.8       & 90.6    & 31.1     & 72.2   & 56.4 & 60.6    & 73.5   \\
Baseline \cite{Radford2021Learning}   & ViT-L/14      & 73.7    & 76.2     & 92.8       & 93.5 & 78.8 & 78.3       & 93.8    & 37.2     & 68.4   & 55.7 & 59.6    & 76.9   \\ 
CoOp \cite{Zhou2022Learning}       & ViT-B/16      & 71.6    & 71.9     & 93.7       & 94.5 & 68.1 & 74.1       & 85.2    & 28.7     & 72.5   & 54.2 & 68.7    & 67.5   \\
CoCoOp \cite{Zhou2022Conditional}     & ViT-B/16      & 75.8    & 73.1     & 95.8       & 96.4 & 72.0 & 81.7       & 91.0    & 27.7     & 78.3   & 64.8 & 71.2    & 77.6   \\
UPL \cite{Huang2022Unsupervised}        & ResNet-50     & 68.4    & 61.1     & 91.4       & 89.5 & 71.0 & 76.6       & 77.9    & 21.7     & 66.4   & 55.1 & 71.0    & 70.2   \\
ProDA \cite{Lu2022Prompt}   & ResNet-50     & -       & 65.3     & 91.3       & 90.0 & 75.5 & 95.5       & 82.4    & 36.6     & -      & 70.1 & 84.3    & -      \\
ProGrad \cite{Zhu2022Prompt}    & ResNet-50     & 67.9    & 62.1     & 91.5       & 93.4 & 62.7 & 78.7       & 81.0    & 21.9     & 70.3   & 57.8 & 59.0    & 68.5   \\
PLOT \cite{Chen2023PLOT}       & ResNet-50     & 73.9    & 63.0     & 92.2       & 87.2 & 72.8 & 94.8       & 77.1    & 34.5     & 70.0   & 65.6 & 82.2    & 77.3   \\
CAVPT \cite{Xing2022Class}      & ViT-B/16      & 83.2    & 72.5     & 96.1       & 93.5 & 88.2 & 97.6       & 85.0    & 57.9     & 74.3   & 72.6 & 92.1    & 85.3   \\
UPT \cite{Zang2022Unified}        & ViT-B/16      & 76.2    & 73.2     & 96.1       & 96.3 & 71.8 & 81.0       & 91.3    & 34.5     & 78.7   & 65.6 & 72.0    & 77.2   \\
TPT \cite{Shu2022Test}        & ViT-B/16      & 64.8    & 69.0     & 94.2       & 87.8 & 66.9 & 69.0       & 84.7    & 24.8     & 65.5   & 47.8 & 42.4    & 60.8   \\
Tip-Adapter \cite{Zhang2021Tip} & ViT-B/16      & -       & 70.8     & -          & -    & -    & -          & -       & -        & -      & -    & -       & -      \\
SgVA-CLIP \cite{Peng2022SgVA}  & ViT-B/16      & -       & 73.3     & -          & -    & -    & -          & -       & -        & 76.4   & -    & -       & -      \\
CALIP \cite{Guo2023CALIP}      & ResNet-50     & 59.4    & 60.6     & 87.7       & 58.6 & 77.4 & 66.4       & 56.3    & 17.7     & 86.2   & 42.4 & 38.9    & 61.7   \\
CuPL \cite{Pratt2022What}       & ViT-L/14      & -       & 76.6     & 93.4       & 93.8 & 77.6 & -          & 93.3    & 36.1     & 61.7   & -    & -       & -      \\
SuS-X \cite{Udandarao2022SuS}      & ResNet-50     & -       & 61.8     & -          & -    & -    & -          & -       & -        & -      & -    & 45.6    & 50.6   \\
SubPT \cite{Ma2023Understanding}       & ResNet-50     & 66.4    & 63.4     & 91.7       & 91.8 & 60.7 & 73.8       & 81.0    & 20.3     & 70.2   & 54.7 & 54.5    & 68.1   \\
VT-Clip \cite{Zhang2021VT}    & ResNet-50     & -       & -        & -          & 93.1 & -    & -          & -       & -        & -      & 65.7 & -       & -      \\ \hline
\end{tabular}}
\end{table*}

We collect the few-shot learning experimental results for selected methods on the 11 datasets from various sources, as shown in Table \ref{tab:TABLE2}. The backbone networks primarily used are CNN-based ResNet50, as well as Transformer-based ViT-B and ViT-L. All methods are trained with only 16 samples per class and then tested for image classification accuracy on the test set. The "Baseline" refers to the classification results of Linear-probe CLIP.

Based on Table \ref{tab:TABLE2}, the following conclusions can be drawn:1)Three multi-modal foundation models' adaptation methods for few-shot learning effectively improve the adaptability of foundation models to downstream tasks under few-shot conditions. Methods such as CoOp based on prompts, SgVA-CLIP based on adapters, and CuPL based on external knowledge show performance improvements of 1.7$\%$ (ViT-B/16), 3.1$\%$ (ViT-B/16), and 0.1$\%$ (ViT-L/14), respectively, on the ImageNet dataset. 2) For prompt-based fine-tuning methods, some unsupervised training methods yield results similar to supervised training methods. The accuracy of the unsupervised method UPL is only 0.4$\%$ higher than that of CoOp trained with 2 samples, while the accuracy of TPT (69.0$\%$) is not significantly different from CoOp trained with 16 samples (71.9$\%$). This is because unsupervised training methods can leverage unlabeled data effectively and can avoid overfitting compared to supervised training with only a small number of samples.

\begin{table}[htb]\centering
\caption{Comparison of base-to-new generalization experimental results}
\label{tab:TABLE3}
\setlength{\tabcolsep}{3mm}{
\renewcommand{\arraystretch}{1.5}
\begin{tabular}{cccccccccccccc}
\cline{1-5}
\multirow{2}{*}{\textbf{Method}} & \multirow{2}{*}{\textbf{Image Encoder}} & \multicolumn{3}{c}{\textbf{Average}} &  &  &  &  &  &  &  &  &  \\ \cline{3-5}
                        &                                & \textbf{Base}    & \textbf{New}     & \textbf{H}       &  &  &  &  &  &  &  &  &  \\ \cline{1-5}
CLIP \cite{Radford2021Learning}  & ViT-B/16                       & 69.34   & 74.22   & 71.70   &  &  &  &  &  &  &  &  &  \\
CoOp \cite{Zhou2022Learning}   & ViT-B/16                       & 82.63   & 67.99   & 74.60   &  &  &  &  &  &  &  &  &  \\
CoCoOp \cite{Zhou2022Conditional}  & ViT-B/16                       & 80.47   & 71.69   & 75.83   &  &  &  &  &  &  &  &  &  \\
ProDA \cite{Lu2022Prompt} & ViT-B/16                       & 81.56   & 72.30   & 76.65   &  &  &  &  &  &  &  &  &  \\
ProGrad \cite{Zhu2022Prompt}   & ViT-B/16                       & 82.48   & 70.75   & 76.16   &  &  &  &  &  &  &  &  &  \\
MAPLE \cite{Khattak2022MaPLe}  & ViT-B/16                       & 82.28   & 75.14   & 78.55   &  &  &  &  &  &  &  &  &  \\
GRAM  \cite{Li2023Gradient}    & ViT-B/16                       & 78.74   & 74.93   & 76.79   &  &  &  &  &  &  &  &  &  \\
KgCoOp \cite{Yao2023Visual}   & ViT-B/16                       & 80.73   & 73.60   & 77.00   &  &  &  &  &  &  &  &  &  \\
CPBPrompt \cite{Liu2023Patch} & ViT-B/16                       & 80.88   & 74.74   & 77.69   &  &  &  &  &  &  &  &  &  \\
UNIGRAM \cite{Li2023Gradient}  & ViT-B/16                       & 80.34   & 75.92   & 78.07   &  &  &  &  &  &  &  &  &  \\
LASP-V  \cite{Bulat2023LASP}   & ViT-B/16                       & 83.18   & 76.11   & 79.48   &  &  &  &  &  &  &  &  &  \\ \cline{1-5}
\end{tabular}}
\end{table}

As shown in Table \ref{tab:TABLE3}, we also collect the experimental results from base-to-new generalization for existing methods from various sources. From the table, we can conclude that: 1) Methods that perform well on the base classes often sacrifice their generalization performance on the new classes. This may be due to the model overfitting to the base classes. Meanwhile, disrupting or forgetting the potential generalization knowledge acquired during pre-training leads to a drop in performance when the model generalizes to unseen classes. 2) LASP-V and CPL perform well in terms of base class discriminability and new class generalization. Across the 11 datasets, LASP-V outperforms CLIP by 13.76$\%$, 1.89$\%$, and 7.78$\%$ for the three metrics of the base class, new class, and harmonic mean, respectively. On the ImageNet dataset, CPL outperforms CLIP by 6.38$\%$, 5.03$\%$, and 5.67$\%$ respectively. Notably, both of these methods exhibit significantly better performance on unseen classes compared to CLIP, which has strong zero-shot capabilities.   

\begin{table*}[htb]
\centering
\caption{Comparison of domain generalization experimental results}
\label{tab:TABLE4}
\setlength{\tabcolsep}{1.4mm}{
\renewcommand{\arraystretch}{1.5}
\begin{tabular}{ccccccccccccccc}
\cline{1-7}
\multirow{2}{*}{\textbf{Method}} & \textbf{Source Domain} &  & \multicolumn{4}{c}{\textbf{Target Domain}}  \\ \cline{2-2} \cline{4-7}
 & \textbf{ImageNet \cite{Deng2009ImageNet}}  &  & \textbf{-V2 \cite{Recht2019Do}} & \textbf{-Sketch \cite{Wang2019Learning}} & \textbf{-A \cite{Hendrycks2021Natural}} & \textbf{-R \cite{Hendrycks2021Many}}  \\ \cline{1-7}
zero shot CLIP \cite{Radford2021Learning} & 66.73 &  & 60.83 & 46.15 & 47.77 & 73.96   \\
linear probe CLIP \cite{Radford2021Learning} & 65.85 &  & 56.26 & 34.77 & 35.68 & 58.43  \\
CoOp \cite{Zhou2022Learning} & 71.51 &  & 64.20 & 47.99 & 49.71 & 75.21 \\
CoCoOp \cite{Zhou2022Conditional}& 71.02 &  & 64.07 & 48.75 & 50.63 & 76.18 \\
VPT-shallow \cite{Jia2022Visual} & 68.98 &  & 62.10 & 47.68 & 47.19 & 76.10 \\
VPT-deep \cite{Jia2022Visual} & 70.57 &  & 63.67 & 47.66 & 43.85 & 74.42 \\
MAPLE \cite{Khattak2022MaPLe} & 70.72 &  & 64.07 & 49.15 & 50.90 & 76.98 \\
UPT \cite{Zang2022Unified} & 72.63 &  & 64.35 & 48.66 & 50.66 & 76.24  \\
KgCoOp \cite{Yao2023Visual}  & 71.20 &  & 64.10 & 48.97 & 50.69 & 76.70  \\
PBPrompt \cite{Liu2023Patch} & 70.90 &  & 64.40 & 49.10 & 51.00 & 76.40  \\
UNIGRAM \cite{Li2023Gradient} & 71.65 &  & 64.81 & 49.54 & 51.51 & 77.34 \\
TPT+CoOp \cite{Shu2022Test} & 73.61 &  & 66.83 & 49.29 & 57.95 & 77.27 \\
TPT+CoCoOp \cite{Shu2022Test} & 71.07 &  & 64.85 & 48.47 & 58.47 & 78.65 \\ \cline{1-7}
\end{tabular}}
\end{table*}

The collected domain generalization experimental results for existing methods from various sources are shown in Table \ref{tab:TABLE4}, and we can find that: 1) TPT can effectively combine with CoOp and CoCoOp and achieve state-of-the-art (SOTA) performance. When using ViT-B/16 as the backbone, the combination of TPT and CoOp outperforms zero-shot CLIP by 6.88$\%$, 6$\%$, 3.14$\%$, 10.18$\%$, and 3.31$\%$ on the ImageNet, -V2, -Sketch, -A, and -R datasets, respectively. 2) Compared to both textual-only prompt-based fine-tuning adaptation methods (e.g., CoOp, CoCoOp) and visual-only prompt-based fine-tuning adaptation methods (e.g., VPT), multi-modal prompt-based fine-tuning adaptation methods like UPT and MAPLE achieve more improvements. This suggests that fine-tuning the two modalities simultaneously is both sufficient and necessary for better performance.

\begin{table*}[htb]
\centering
\caption{Comparison of cross-dataset transfer experimental results}
\label{tab:TABLE5}
\resizebox{\textwidth}{!}{
\renewcommand{\arraystretch}{1.5}
\begin{tabular}{ccclccccccccccc}
\hline
\multirow{2}{*}{\textbf{\thead{\\ \\ \\ Method}}} & \multirow{2}{*}{\textbf{\thead{\\ \\ Image \\ Encoder \\ Structure}}} & \textbf{\thead{Source \\ Dataset}} &  & \multicolumn{11}{c}{\textbf{\thead{Target \\ Dataset}}} \\ \cline{3-3} \cline{5-15} 
 &  &  \rotatebox{90}{\begin{tabular}[c]{@{}l@{}}\textbf{ImageNet \cite{Deng2009ImageNet}}\end{tabular}} &  & \rotatebox{90}{\begin{tabular}[c]{@{}l@{}}\textbf{Caltech101 \cite{FeiFei2007Learning}}\end{tabular}} & \rotatebox{90}{\begin{tabular}[c]{@{}l@{}}\textbf{Pets \cite{Parkhi2012Cats}}\end{tabular}} & \rotatebox{90}{\begin{tabular}[c]{@{}l@{}}\textbf{Cars \cite{Krause20133D}}\end{tabular}} & \rotatebox{90}{\begin{tabular}[c]{@{}l@{}}\textbf{Flowers102 \cite{Nilsback2008Automated}}\end{tabular}} & \rotatebox{90}{\begin{tabular}[c]{@{}l@{}}\textbf{Food101 \cite{Bossard2014Food}}\end{tabular}} & \rotatebox{90}{\begin{tabular}[c]{@{}l@{}}\textbf{Aircraft \cite{Maji2013Fine}}\end{tabular}} & \rotatebox{90}{\begin{tabular}[c]{@{}l@{}}\textbf{SUN397 \cite{Xiao2010SUN}}\end{tabular}} & \rotatebox{90}{\begin{tabular}[c]{@{}l@{}}\textbf{DTD \cite{Cimpoi2014Describing}}\end{tabular}} & \rotatebox{90}{\begin{tabular}[c]{@{}l@{}}\textbf{EuroSAT \cite{Helber2019EuroSAT}}\end{tabular}} & \rotatebox{90}{\begin{tabular}[c]{@{}l@{}}\textbf{UCF101 \cite{Soomro2012UCF101}}\end{tabular}} & \rotatebox{90}{\begin{tabular}[c]{@{}l@{}}\textbf{Average}\end{tabular}} \\ \hline
CoOp \cite{Zhou2022Learning} & ViT-B/16 & 71.5 &  & 93.7 & 89.1 & 64.5 & 68.7 & 85.3 & 18.5 & 64.2 & 41.9 & 46.4 & 66.6 & 63.9 \\
CoCoOp \cite{Zhou2022Conditional} & ViT-B/16 & 71.0 &  & 94.4 & 90.1 & 65.3 & 71.9 & 86.1 & 22.9 & 64.4 & 45.7 & 45.4 & 68.2 & 65.7 \\
MAPLE \cite{Khattak2022MaPLe} & ViT-B/16 & 70.7 &  & 95.5 & 90.5 & 65.6 & 72.2 & 86.2 & 24.7 & 67.0 & 46.5 & 48.1 & 68.7 & 66.3 \\
SubPT \cite{Ma2023Understanding} & ResNet-50 & 62.6 &  & 88.3 & 87.4 & 56.2 & 63.4 & 77.8 & 16.7 & 61.8 & 39.7 & 29.1 & 61.7 & 58.6 \\
VPT \cite{Jia2022Visual} & ViT-L/14 & 70.6 &  & 95.8 & 92.9 & 76.1 & 95.0 & 86.2 & 41.0 & 71.6 & 69.8 & 91.5 & 82.8 & 79.4 \\
UPT \cite{Zang2022Unified} & ViT-L/14 & 72.6 &  & 95.9 & 93.0 & 84.3 & 97.1 & 85.0 & 46.8 & 75.9 & 70.7 & 90.5 & 84.0 & 81.4 \\
TPT \cite{Shu2022Test} & ResNet-50 & 60.7 &  & 87.0 & 84.5 & 58.5 & 62.7 & 74.9 & 17.6 & 61.5 & 40.8 & 28.3 & 60.8 & 57.7 \\
UNIGRAM \cite{Li2023Gradient}& ViT-B/16 & 71.7 &  & 94.7 & 90.8 & 66.8 & 73.1 & 86.7 & 25.3 & 68.0 & 48.1 & 52.6 & 71.0 & 67.7 \\
CPBPrompt \cite{Liu2023Patch} & ViT-B/16 & 70.9 &  & 94.9 & 90.8 & 65.3 & 72.4 & 86.4 & 24.6 & 67.8 & 45.2 & 45.1 & 68.8 & 66.1 \\ \hline
\end{tabular}}
\end{table*}

The collected results of cross-dataset adaptation experiments for existing methods from various sources are shown in Table \ref{tab:TABLE5}. The purpose of this experiment is to validate the generalization performance of adaptation methods across different datasets. Using the ImageNet dataset with 1000 classes as the source dataset, the methods are initially trained with 16 samples from each class of the source dataset. Subsequently, these methods are tested on 10 different target datasets. 

From the data in the table, the following observations can be made: 1) For SubPT and TPT with ResNet50 as the backbone, as well as other methods with ViT-based backbones, they achieve similar results on the source dataset but exhibit significant variations in performance on different target datasets. For example, they achieve higher accuracy on datasets like Caltech101 and Pets, which have categories similar to ImageNet. However, their performances are much lower on datasets like Aircraft and DTD, which contain fine-grained data related to various aircraft and textures. These datasets have greater category differences from ImageNet, and hence, the accuracies on these datasets are much lower than 50$\%$, indicating that transferring specific knowledge learned from ImageNet to downstream tasks with significantly different categories is challenging. 2) Whether it is prompt-based fine-tuning methods or adapter-based fine-tuning methods, fine-tuning both modalities simultaneously tends to yield better results than fine-tuning only one modality. For instance, the multi-modal prompt learning method MAPLE achieves higher accuracies on 10 target datasets compared to the textual-only prompt learning method CoCoOp (ViT-B/16), and UPT achieves higher accuracies compared to the visual-only prompt learning method VPT (ViT-L/14). This suggests that for multi-modal foundation models like CLIP, fine-tuning both textual and visual aspects is essential for improved generalization performance.

\section{Analysis and Discussion} 

The current researches on few-shot adaptation for multi-modal foundation models primarily include prompt-based fine-tuning adaptation methods, adapter-based fine-tuning adaptation methods, and external knowledge-based adaptation methods. Based on the current research status on few-shot adaptation, we summarise the following issues and challenges:

\textbf{1) Ineffective Adaptation of Upstream and Downstream Domain Distributions:} Existing few-shot adaptation methods for multi-modal foundation models mostly focus on the category information in downstream task data while neglecting domain distribution information. Additionally, in the domain adaptation scenarios of multi-modal foundation models, the scarcity of samples in the target domain makes modeling challenging, while the abundance of samples in the source domain results in high modeling costs. Furthermore, current domain adaptation methods are tailored for a single modality, ignoring cross-modal information interaction, which does not meet the requirements of adaptation for multi-modal foundation models. 

\textbf{2) Lack of Adaptability in Model Selection:} There is a wide variety of existing adapter structures with different fine-tuning characteristics, learning capabilities, and model capacities. Different downstream tasks often require different optimal adapters or combinations of multiple adapters, and the feature extraction processes vary significantly across different modalities. Currently, adapter selection is typically reliant on heuristic methods or exhaustive search-based approaches, which are costly and challenging to guarantee performance.

\textbf{3) Insufficient Utilization of Data and Knowledge:} Although existing data augmentation strategies can enrich the downstream training set and reduce the risk of model overfitting to some extent. However, this empirically manually designed augmentation approach is costly and does not ensure that the selected data augmentation can be effectively adapted to the specific downstream task. Although new textual descriptions can be generated by introducing external knowledge, there is no guarantee that the textual knowledge describing the category attributes can effectively collaborate with the image data.

To address the above-mentioned issues and challenges, we summarize and refine the existing theories related to cross-domain adaptation for multi-modal foundation models, which makes the work of few-shot adaptation more systematic and guides few-shot cross-domain adaptation.



Maurer et al. \cite{Maurer2004Note} focused on the relationship between mean error and expected error in the target domain in 2004:

\begin{equation}\label{lemma1}
\begin{aligned}
\Delta _{h_{N}}\left( T,\mathbb{T} \right)\le\sqrt{\frac{KL(h_N||h_0)+ln{\sqrt{4N}}-ln{\left(\delta\right)}}{2N}}.
\end{aligned}
\end{equation} 
where $\Delta _{h_{N}}\left( T,\mathbb{T} \right)=\left|\epsilon_T\left(h_N\right)-\epsilon_\mathbb{T}\left(h_N\right)\right|$, $\epsilon_T\left(h_N\right)$ denotes the mean error in the target domain, $\epsilon_\mathbb{T}\left(h_N\right)$ represents the expected error in the target domain, $T$ stands for the data in the target domain dataset, $\mathbb{T}$ represents all the possible data that may exist in the target domain, while $h_0$ represents the prior model, $h_N$ represents the adapted model, and $N$ represents the number of samples used from the target domain dataset.

Anthony Sicilia et al. \cite{Sicilia2022PAC} derived an empirical calculation of the domain gap from source and target domain data:

\begin{equation}\label{lemma2}
\begin{aligned}
\Delta _{h_{N}}\left( S,T \right)\le\widetilde{\lambda_{S,T}}+E\left[d\left(S,T\right)\right].
\end{aligned}
\end{equation} 
where $\widetilde{\lambda_{S,T}}$ denotes the minimum value of the model's error sum over the source and target domains in the fine-tuning space to indicate the adaptability of upstream and downstream tasks, and $E[d(S,T)]$ denotes the domain gap as the H-divergence between two distinct domains.

Crammer et al. \cite{Crammer2006Learning} proposed the triangle inequality for errors:

\begin{equation}\label{lemma3}
\begin{aligned}
\Delta _{h_{N}}\left( S,\mathbb{T} \right)\le\Delta _{h_{N}}\left( S,T \right)+\Delta _{h_{N}}\left( T,\mathbb{T} \right).
\end{aligned}
\end{equation} 

Based on the above lemmas, Anthony Sicilia et al. proposed the PAC-Bayesian domain adaptation bound theory for multi-class classifiers:

\begin{equation}\label{lemma4}
\begin{aligned}
\epsilon_\mathbb{T}\left(h_N\right)\le\widetilde{\lambda_{S,T}}+\epsilon_S\left(h_N\right)+E\left[d\left(S,T\right)\right]+\\\sqrt{\frac{KL(h_N||h_0)+ln{\sqrt{4N}}-ln{\left(\delta\right)}}{2N}}.
\end{aligned}
\end{equation} 
however, the term $\epsilon_S(h_N)$ in the theory represents the empirical error of the fine-tuning model on the source domain data, which is independent of the adaptation to the downstream target domain. Conversely, the empirical error $\epsilon_T(h_N)$ on the target domain, which affects the adaptation, is not captured. To make the theory better represent the impact of target domain data on the adaptation, we propose the following theorem.

\newtheorem{thm}{\bf Theorem}
\begin{thm}\label{thm1}
The expected error $\in_\mathbb{T}\left(h_{N}\right)$ of the adapted fine-tuned model in the target domain is defined, with the empirical error of the fine-tuned model in the target domain denoted as $\in_{T}\left(h_{N}\right)$. The adaptability of the upstream and downstream tasks is defined as $\widetilde{\lambda_{S,T}}$, with $E[d(S, T)]$ representing the gap between the source and target domains, and $KL(h_N\ ||\ h_0)$ indicating the difference between the original model and the fine-tuned model. Here, N represents the number of data samples. Therefore, the expected error bound $\in_\mathbb{T}\left(h_{N}\right)$ of the fine-tuned model in the target domain, for adaptation, depends on its empirical error $\in_{T}\left(h_{N}\right)$, the domain difference $E[d(S, T)]$ between the source and target domains, model capacity $KL(h_N\ ||\ h_0)$, sample size $N$, and the adaptability of the upstream and downstream tasks ${\widetilde{\lambda}}_{S, T}$:

\begin{equation}
\begin{aligned}
\epsilon_\mathbb{T}\left(h_N\right)\le\epsilon_T\left(h_N\right)+2E\left[d\left(S,T\right)\right]+\\\sqrt{\frac{KL(h_N||h_0)+ln{\sqrt{4N}}-ln{\left(\delta\right)}}{2N}}+2\widetilde{\lambda_{S,T}}.
\end{aligned}
\end{equation} 

\end{thm} 

\begin{proof}
We choose to replace $\epsilon_S(h_N)$ in (\ref{lemma4}) with $\epsilon_S\left(h_N\right)-\epsilon_T\left(h_N\right)+\epsilon_T\left(h_N\right)$, which is organized into the form of (\ref{equation6}):

\begin{equation}\label{equation6}
\begin{aligned}
\epsilon_\mathbb{T}\left(h_N\right)\le\widetilde{\lambda_{S,T}}+\epsilon_S\left(h_N\right)-\epsilon_T\left(h_N\right)+\epsilon_T\left(h_N\right)+\\E\left[d\left(S,T\right)\right]+\sqrt{\frac{KL(h_N||h_0)+ln{\sqrt{4N}}-ln{\left(\delta\right)}}{2N}}.
\end{aligned}
\end{equation} 
as $\epsilon_S\left(h_N\right)-\epsilon_T\left(h_N\right)\le{|\epsilon}_S\left(h_N\right)-\epsilon_T\left(h_N\right)|$ is constant, we can derive that:

\begin{equation}\label{equation7}
\begin{aligned}
\epsilon_\mathbb{T}\left(h_N\right)\le\widetilde{\lambda_{S,T}}+{|\epsilon}_S\left(h_N\right)-\epsilon_T\left(h_N\right)|+\epsilon_T\left(h_N\right)+\\ E\left[d\left(S,T\right)\right]+\sqrt{\frac{KL(h_N||h_0)+ln{\sqrt{4N}}-ln{\left(\delta\right)}}{2N}}.
\end{aligned}
\end{equation} 
according to the definition of $\Delta _{h_{N}}\left( S,T \right)={|\epsilon}_S\left(h_N\right)-\epsilon_T\left(h_N\right)|$, it can be obtained:

\begin{equation}\label{equation8}
\begin{aligned}
\epsilon_\mathbb{T}\left(h_N\right)\le\widetilde{\lambda_{S,T}}+\Delta _{h_{N}}\left( S,T \right)+\epsilon_T\left(h_N\right)+\\E\left[d\left(S,T\right)\right]+\sqrt{\frac{KL(h_N||h_0)+ln{\sqrt{4N}}-ln{\left(\delta\right)}}{2N}}.
\end{aligned}
\end{equation} 
substituting (\ref{lemma2}) into (\ref{equation8}), the collation gives:

\begin{equation}\label{equation9}
\begin{aligned}
\epsilon_\mathbb{T}\left(h_N\right)\le\epsilon_T\left(h_N\right)+2E\left[d\left(S,T\right)\right]+\\\sqrt{\frac{KL(h_N||h_0)+ln{\sqrt{4N}}-ln{\left(\delta\right)}}{2N}}+2\widetilde{\lambda_{S,T}}.
\end{aligned}
\end{equation} 

\end{proof}

The adaptability of the upstream and downstream tasks mentioned in the above theorem is determined by the nature of the tasks themselves. Once the tasks are fixed, this factor remains constant. However, by adjusting domain discrepancies, model capacity, and sample sizes, it is possible to enhance the model's generalization performance and reduce empirical errors. Therefore, domain discrepancies between the source and target domains, model capacity, and sample sizes are three fundamental factors that influence the adaptation of multi-modal foundation models.
\begin{figure*}[htp]
    \centering\includegraphics[width=0.98\textwidth]{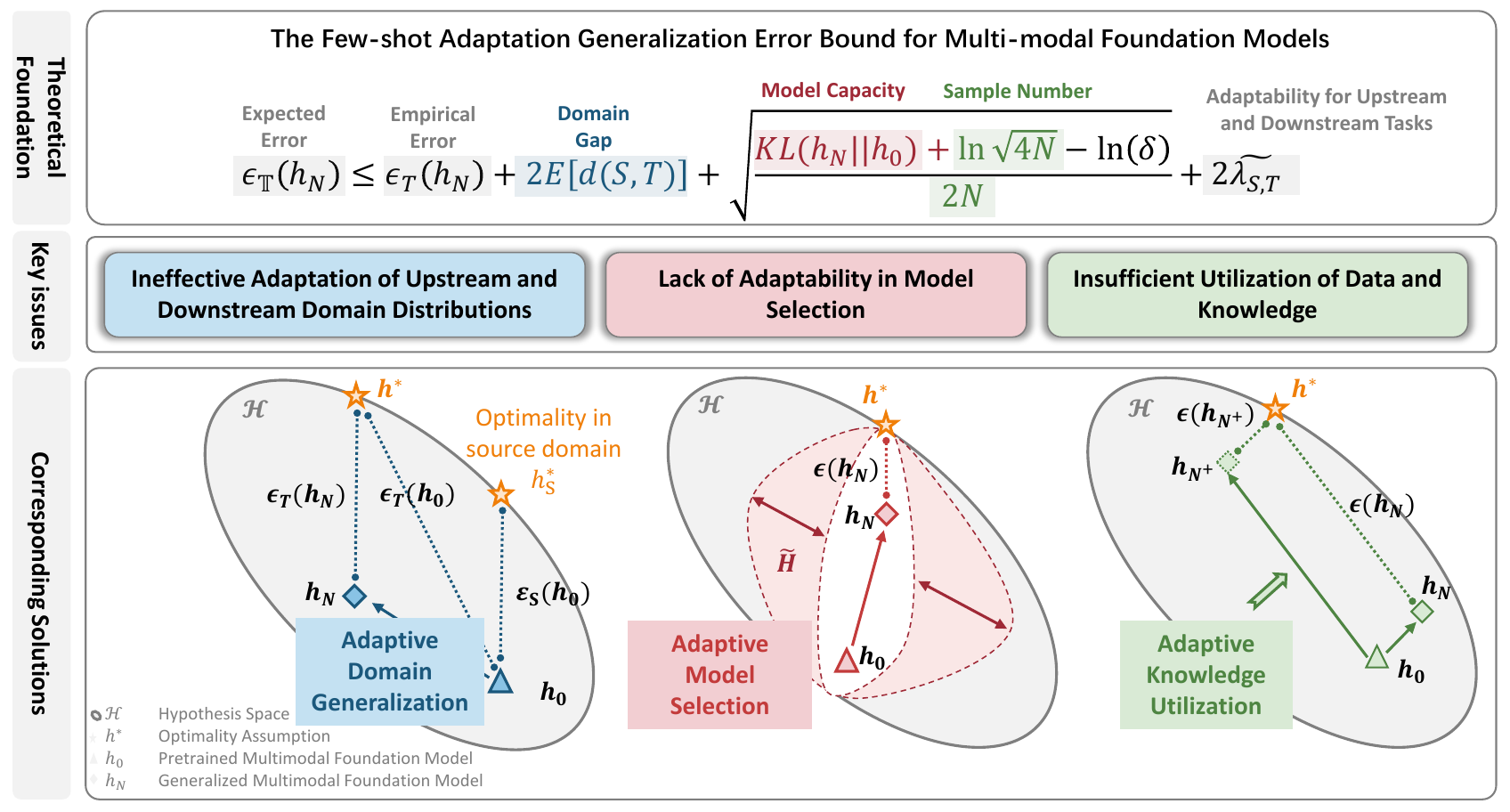}
    \caption{Key issues of the few-shot adaptation methods for multi-modal foundation models and corresponding solutions.}
    \label{fig:fig4}
    \end{figure*}
    
Taking guidance from the theory of generalization error bound in few-shot cross-domain adaptation, We start from three aspects, namely: adaptive domain generalization, adaptive model selection , and adaptive knowledge utilization, as shown in Figure \ref{fig:fig4}, to study the few-shot adaptation methods for multi-modal foundation models and propose corresponding solutions to the aforementioned problems:
    
\textbf{1) Adaptive domain generalization:} To address the issue of high modeling cost in domain adaptation, a possible approach is to consider source-free domain adaptation methods for multi-modal foundation models. The reconstruction error of a pre-trained autoencoder-based model can be used as a measure of domain distribution discrepancy and the advantages of prompt-based fine-tuning adaptation methods can also be combined, which can obtain more convenient and efficient adaptation methods. Additionally, to avoid the risk of modality gap during domain alignment and the loss of cross-modal semantic relevance, a multi-modal autoencoder can be introduced in the prompt reconstruction process to constrain the cross-modal joint distribution of data. This will help maintain the semantic consistency of textual and visual features during the domain adaptation process.

\textbf{2) Adaptive model selection :} Neural Architecture Search (NAS) \cite{Kang2023Neural} is a promising approach to address the problem of adaptive selection of adapter structures in multi-modal foundation models. It automatically explores different network architectures in the search space to find the best-performing structure for a given task. However, due to the complexity of multi-modal foundation model structures, NAS-based methods need to search for a wide variety of adapter types simultaneously, resulting in a large search space and high computational costs. In such cases, it is necessary to adopt a coarse-to-fine search strategy to design more efficient NAS-based search methods. 

\textbf{3) Adaptive knowledge utilization:} Traditional image augmentation techniques can generate a large number of augmented samples but may not effectively adapt to specific downstream tasks. In contrast, the continuous differentiable image augmentation paradigm can solve for optimal image augmentation parameters through derivation and backpropagation. Therefore, it is worth exploring a differentiable image augmentation approach to achieve adaptive image augmentation for downstream tasks. Additionally, introducing external knowledge can provide richer textual descriptions for multi-modal foundation models. However, it is essential to ensure that the introduced knowledge is highly relevant to visual semantics. One approach could involve using a large language model to generate reliable visual descriptions as references and then training visual filters through adversarial learning to ensure that the filtered textual descriptions contain valid visual semantics.
    
Due to the diversity of downstream tasks in the real world, they vary in domain distribution, task attributes, available sample quantities, and so on. Current adaptation methods of foundation models do not possess the capability to adapt to these factors effectively, resulting in limited model performance. This limitation has become a bottleneck hindering the further adoption of foundation models in various industries. Therefore, endowing multi-modal foundation models with adaptability in the context of downstream few-shot adaptation is crucial. This can be achieved through adaptive domain generalization, adaptive model selection , and adaptive knowledge utilization. These adaptations have the potential to significantly improve model performance and may represent important research directions in this field in the future.

\section{Conclusion} 

We have comprehensively summarized the methods for multi-modal foundation models in the context of few-shot adaptation tasks, including prompt-based fine-tuning adaptation, adapter-based fine-tuning adaptation, and external knowledge-based fine-tuning adaptation. Prompt-based fine-tuning adaptation methods avoid the tediousness of manually designing the textual prompt and require only a small number of parameters to be fine-tuned, thus effectively mitigating the overfitting problem. Adapter-based fine-tuning adaptation methods only need to update a small number of parameters and have the advantages of high efficiency, design flexibility, and robustness. External knowledge-based fine-tuning adaptation methods can alleviate the problem of insufficient prior knowledge and scarce training samples in downstream few-shot scenarios to a certain extent. However, these methods still have some limitations, such as ineffective adaptation of upstream and downstream domain distributions, lack of adaptability in model selection, and insufficient utilization of data and knowledge. Therefore, we believe that in the future, we need to take three perspectives: adaptive domain generalization, adaptive model selection, and adaptive knowledge utilization in order to improve the performance in multi-modal few-shot adaptation. In addition, we review 11 commonly used datasets for evaluating the downstream generalization performance of multi-modal foundation models and adopt four experimental setups to test the generalization performance of multi-modal foundation models under few-shot conditions. We hope that the summary and analysis of this survey can provide some insights and guidance for future research on the few-shot adaptation of multi-modal foundation models.

\bibliographystyle{IEEEtran}
\bibliography{IEEEabrv,few_shot_survey}

\begin{thebibliography}{100}
\providecommand{\url}[1]{#1}
\csname url@samestyle\endcsname
\providecommand{\newblock}{\relax}
\providecommand{\bibinfo}[2]{#2}
\providecommand{\BIBentrySTDinterwordspacing}{\spaceskip=0pt\relax}
\providecommand{\BIBentryALTinterwordstretchfactor}{4}
\providecommand{\BIBentryALTinterwordspacing}{\spaceskip=\fontdimen2\font plus
\BIBentryALTinterwordstretchfactor\fontdimen3\font minus
  \fontdimen4\font\relax}
\providecommand{\BIBforeignlanguage}[2]{{%
\expandafter\ifx\csname l@#1\endcsname\relax
\typeout{** WARNING: IEEEtran.bst: No hyphenation pattern has been}%
\typeout{** loaded for the language `#1'. Using the pattern for}%
\typeout{** the default language instead.}%
\else
\language=\csname l@#1\endcsname
\fi
#2}}
\providecommand{\BIBdecl}{\relax}
\BIBdecl

\bibitem{Devlin2019BERT}
\BIBentryALTinterwordspacing
J.~Devlin, M.~Chang, K.~Lee, and K.~Toutanova, ``{BERT:} pre-training of deep
  bidirectional transformers for language understanding,'' in \emph{Proceedings
  of the 2019 Conference of the North American Chapter of the Association for
  Computational Linguistics: Human Language Technologies, {NAACL-HLT} 2019,
  Minneapolis, MN, USA, June 2-7, 2019, Volume 1 (Long and Short Papers)},
  J.~Burstein, C.~Doran, and T.~Solorio, Eds.\hskip 1em plus 0.5em minus
  0.4em\relax Association for Computational Linguistics, 2019, pp. 4171--4186.
  [Online]. Available: \url{https://doi.org/10.18653/v1/n19-1423}
\BIBentrySTDinterwordspacing

\bibitem{Zeng2021PanGu}
\BIBentryALTinterwordspacing
W.~Zeng, X.~Ren, T.~Su, H.~Wang, Y.~Liao, Z.~Wang, X.~Jiang, Z.~Yang, K.~Wang,
  X.~Zhang, C.~Li, Z.~Gong, Y.~Yao, X.~Huang, J.~Wang, J.~Yu, Q.~Guo, Y.~Yu,
  Y.~Zhang, J.~Wang, H.~Tao, D.~Yan, Z.~Yi, F.~Peng, F.~Jiang, H.~Zhang,
  L.~Deng, Y.~Zhang, Z.~Lin, C.~Zhang, S.~Zhang, M.~Guo, S.~Gu, G.~Fan,
  Y.~Wang, X.~Jin, Q.~Liu, and Y.~Tian, ``Pangu-{\(\alpha\)}: Large-scale
  autoregressive pretrained chinese language models with auto-parallel
  computation,'' \emph{CoRR}, vol. abs/2104.12369, 2021. [Online]. Available:
  \url{https://arxiv.org/abs/2104.12369}
\BIBentrySTDinterwordspacing

\bibitem{Chowdhery2022PaLM}
\BIBentryALTinterwordspacing
A.~Chowdhery, S.~Narang, J.~Devlin, M.~Bosma, G.~Mishra, A.~Roberts, P.~Barham,
  H.~W. Chung, C.~Sutton, S.~Gehrmann, P.~Schuh, K.~Shi, S.~Tsvyashchenko,
  J.~Maynez, A.~Rao, P.~Barnes, Y.~Tay, N.~Shazeer, V.~Prabhakaran, E.~Reif,
  N.~Du, B.~Hutchinson, R.~Pope, J.~Bradbury, J.~Austin, M.~Isard,
  G.~Gur{-}Ari, P.~Yin, T.~Duke, A.~Levskaya, S.~Ghemawat, S.~Dev,
  H.~Michalewski, X.~Garcia, V.~Misra, K.~Robinson, L.~Fedus, D.~Zhou,
  D.~Ippolito, D.~Luan, H.~Lim, B.~Zoph, A.~Spiridonov, R.~Sepassi, D.~Dohan,
  S.~Agrawal, M.~Omernick, A.~M. Dai, T.~S. Pillai, M.~Pellat, A.~Lewkowycz,
  E.~Moreira, R.~Child, O.~Polozov, K.~Lee, Z.~Zhou, X.~Wang, B.~Saeta,
  M.~Diaz, O.~Firat, M.~Catasta, J.~Wei, K.~Meier{-}Hellstern, D.~Eck, J.~Dean,
  S.~Petrov, and N.~Fiedel, ``Palm: Scaling language modeling with pathways,''
  \emph{CoRR}, vol. abs/2204.02311, 2022. [Online]. Available:
  \url{https://doi.org/10.48550/arXiv.2204.02311}
\BIBentrySTDinterwordspacing

\bibitem{OpenAI2023GPT}
\BIBentryALTinterwordspacing
OpenAI, ``{GPT-4} technical report,'' \emph{CoRR}, vol. abs/2303.08774, 2023.
  [Online]. Available: \url{https://doi.org/10.48550/arXiv.2303.08774}
\BIBentrySTDinterwordspacing

\bibitem{Lu2019ViLBERT}
\BIBentryALTinterwordspacing
J.~Lu, D.~Batra, D.~Parikh, and S.~Lee, ``Vilbert: Pretraining task-agnostic
  visiolinguistic representations for vision-and-language tasks,'' in
  \emph{Advances in Neural Information Processing Systems 32: Annual Conference
  on Neural Information Processing Systems 2019, NeurIPS 2019, December 8-14,
  2019, Vancouver, BC, Canada}, H.~M. Wallach, H.~Larochelle, A.~Beygelzimer,
  F.~d'Alch{\'{e}}{-}Buc, E.~B. Fox, and R.~Garnett, Eds., 2019, pp. 13--23.
  [Online]. Available:
  \url{https://proceedings.neurips.cc/paper/2019/hash/c74d97b01eae257e44aa9d5bade97baf-Abstract.html}
\BIBentrySTDinterwordspacing

\bibitem{Radford2021Learning}
\BIBentryALTinterwordspacing
A.~Radford, J.~W. Kim, C.~Hallacy, A.~Ramesh, G.~Goh, S.~Agarwal, G.~Sastry,
  A.~Askell, P.~Mishkin, J.~Clark, G.~Krueger, and I.~Sutskever, ``Learning
  transferable visual models from natural language supervision,'' in
  \emph{Proceedings of the 38th International Conference on Machine Learning,
  {ICML} 2021, 18-24 July 2021, Virtual Event}, ser. Proceedings of Machine
  Learning Research, M.~Meila and T.~Zhang, Eds., vol. 139.\hskip 1em plus
  0.5em minus 0.4em\relax {PMLR}, 2021, pp. 8748--8763. [Online]. Available:
  \url{http://proceedings.mlr.press/v139/radford21a.html}
\BIBentrySTDinterwordspacing

\bibitem{Li2022Supervision}
\BIBentryALTinterwordspacing
Y.~Li, F.~Liang, L.~Zhao, Y.~Cui, W.~Ouyang, J.~Shao, F.~Yu, and J.~Yan,
  ``Supervision exists everywhere: {A} data efficient contrastive
  language-image pre-training paradigm,'' in \emph{The Tenth International
  Conference on Learning Representations, {ICLR} 2022, Virtual Event, April
  25-29, 2022}.\hskip 1em plus 0.5em minus 0.4em\relax OpenReview.net, 2022.
  [Online]. Available: \url{https://openreview.net/forum?id=zq1iJkNk3uN}
\BIBentrySTDinterwordspacing

\bibitem{Yao2022FILIP}
\BIBentryALTinterwordspacing
L.~Yao, R.~Huang, L.~Hou, G.~Lu, M.~Niu, H.~Xu, X.~Liang, Z.~Li, X.~Jiang, and
  C.~Xu, ``{FILIP:} fine-grained interactive language-image pre-training,'' in
  \emph{The Tenth International Conference on Learning Representations, {ICLR}
  2022, Virtual Event, April 25-29, 2022}.\hskip 1em plus 0.5em minus
  0.4em\relax OpenReview.net, 2022. [Online]. Available:
  \url{https://openreview.net/forum?id=cpDhcsEDC2}
\BIBentrySTDinterwordspacing

\bibitem{Gao2022PyramidCLIP}
\BIBentryALTinterwordspacing
Y.~Gao, J.~Liu, Z.~Xu, J.~Zhang, K.~Li, R.~Ji, and C.~Shen, ``Pyramidclip:
  Hierarchical feature alignment for vision-language model pretraining,'' in
  \emph{NeurIPS}, 2022. [Online]. Available:
  \url{http://papers.nips.cc/paper\_files/paper/2022/hash/e9882f7f7c44a10acc01132302bac9d8-Abstract-Conference.html}
\BIBentrySTDinterwordspacing

\bibitem{Cai2020Once}
\BIBentryALTinterwordspacing
H.~Cai, C.~Gan, T.~Wang, Z.~Zhang, and S.~Han, ``Once-for-all: Train one
  network and specialize it for efficient deployment,'' in \emph{8th
  International Conference on Learning Representations, {ICLR} 2020, Addis
  Ababa, Ethiopia, April 26-30, 2020}.\hskip 1em plus 0.5em minus 0.4em\relax
  OpenReview.net, 2020. [Online]. Available:
  \url{https://openreview.net/forum?id=HylxE1HKwS}
\BIBentrySTDinterwordspacing

\bibitem{Wang2022Image}
\BIBentryALTinterwordspacing
W.~Wang, H.~Bao, L.~Dong, J.~Bjorck, Z.~Peng, Q.~Liu, K.~Aggarwal, O.~K.
  Mohammed, S.~Singhal, S.~Som, and F.~Wei, ``Image as a foreign language: Beit
  pretraining for all vision and vision-language tasks,'' \emph{CoRR}, vol.
  abs/2208.10442, 2022. [Online]. Available:
  \url{https://doi.org/10.48550/arXiv.2208.10442}
\BIBentrySTDinterwordspacing

\bibitem{Shan2022ERNIE}
\BIBentryALTinterwordspacing
B.~Shan, W.~Yin, Y.~Sun, H.~Tian, H.~Wu, and H.~Wang, ``Ernie-vil 2.0:
  Multi-view contrastive learning for image-text pre-training,'' \emph{CoRR},
  vol. abs/2209.15270, 2022. [Online]. Available:
  \url{https://doi.org/10.48550/arXiv.2209.15270}
\BIBentrySTDinterwordspacing

\bibitem{Baevski2022data2vec}
\BIBentryALTinterwordspacing
A.~Baevski, W.~Hsu, Q.~Xu, A.~Babu, J.~Gu, and M.~Auli, ``data2vec: {A} general
  framework for self-supervised learning in speech, vision and language,'' in
  \emph{International Conference on Machine Learning, {ICML} 2022, 17-23 July
  2022, Baltimore, Maryland, {USA}}, ser. Proceedings of Machine Learning
  Research, K.~Chaudhuri, S.~Jegelka, L.~Song, C.~Szepesv{\'{a}}ri, G.~Niu, and
  S.~Sabato, Eds., vol. 162.\hskip 1em plus 0.5em minus 0.4em\relax {PMLR},
  2022, pp. 1298--1312. [Online]. Available:
  \url{https://proceedings.mlr.press/v162/baevski22a.html}
\BIBentrySTDinterwordspacing

\bibitem{Zhou2022Learning}
\BIBentryALTinterwordspacing
K.~Zhou, J.~Yang, C.~C. Loy, and Z.~Liu, ``Learning to prompt for
  vision-language models,'' \emph{Int. J. Comput. Vis.}, vol. 130, no.~9, pp.
  2337--2348, 2022. [Online]. Available:
  \url{https://doi.org/10.1007/s11263-022-01653-1}
\BIBentrySTDinterwordspacing

\bibitem{Gao2021CLIP}
\BIBentryALTinterwordspacing
P.~Gao, S.~Geng, R.~Zhang, T.~Ma, R.~Fang, Y.~Zhang, H.~Li, and Y.~Qiao,
  ``Clip-adapter: Better vision-language models with feature adapters,''
  \emph{CoRR}, vol. abs/2110.04544, 2021. [Online]. Available:
  \url{https://arxiv.org/abs/2110.04544}
\BIBentrySTDinterwordspacing

\bibitem{Pratt2022What}
\BIBentryALTinterwordspacing
S.~M. Pratt, R.~Liu, and A.~Farhadi, ``What does a platypus look like?
  generating customized prompts for zero-shot image classification,''
  \emph{CoRR}, vol. abs/2209.03320, 2022. [Online]. Available:
  \url{https://doi.org/10.48550/arXiv.2209.03320}
\BIBentrySTDinterwordspacing

\bibitem{Radford2018Improving}
\BIBentryALTinterwordspacing
A.~Radford and K.~Narasimhan, ``Improving language understanding by generative
  pre-training,'' 2018. [Online]. Available:
  \url{https://api.semanticscholar.org/CorpusID:49313245}
\BIBentrySTDinterwordspacing

\bibitem{Chen2020Simple}
\BIBentryALTinterwordspacing
T.~Chen, S.~Kornblith, M.~Norouzi, and G.~E. Hinton, ``A simple framework for
  contrastive learning of visual representations,'' in \emph{Proceedings of the
  37th International Conference on Machine Learning, {ICML} 2020, 13-18 July
  2020, Virtual Event}, ser. Proceedings of Machine Learning Research, vol.
  119.\hskip 1em plus 0.5em minus 0.4em\relax {PMLR}, 2020, pp. 1597--1607.
  [Online]. Available: \url{http://proceedings.mlr.press/v119/chen20j.html}
\BIBentrySTDinterwordspacing

\bibitem{He2020Momentum}
\BIBentryALTinterwordspacing
K.~He, H.~Fan, Y.~Wu, S.~Xie, and R.~B. Girshick, ``Momentum contrast for
  unsupervised visual representation learning,'' in \emph{2020 {IEEE/CVF}
  Conference on Computer Vision and Pattern Recognition, {CVPR} 2020, Seattle,
  WA, USA, June 13-19, 2020}.\hskip 1em plus 0.5em minus 0.4em\relax Computer
  Vision Foundation / {IEEE}, 2020, pp. 9726--9735. [Online]. Available:
  \url{https://doi.org/10.1109/CVPR42600.2020.00975}
\BIBentrySTDinterwordspacing

\bibitem{He2022Masked}
\BIBentryALTinterwordspacing
K.~He, X.~Chen, S.~Xie, Y.~Li, P.~Doll{\'{a}}r, and R.~B. Girshick, ``Masked
  autoencoders are scalable vision learners,'' in \emph{{IEEE/CVF} Conference
  on Computer Vision and Pattern Recognition, {CVPR} 2022, New Orleans, LA,
  USA, June 18-24, 2022}.\hskip 1em plus 0.5em minus 0.4em\relax {IEEE}, 2022,
  pp. 15\,979--15\,988. [Online]. Available:
  \url{https://doi.org/10.1109/CVPR52688.2022.01553}
\BIBentrySTDinterwordspacing

\bibitem{Bao2022BEiT}
\BIBentryALTinterwordspacing
H.~Bao, L.~Dong, S.~Piao, and F.~Wei, ``Beit: {BERT} pre-training of image
  transformers,'' in \emph{The Tenth International Conference on Learning
  Representations, {ICLR} 2022, Virtual Event, April 25-29, 2022}.\hskip 1em
  plus 0.5em minus 0.4em\relax OpenReview.net, 2022. [Online]. Available:
  \url{https://openreview.net/forum?id=p-BhZSz59o4}
\BIBentrySTDinterwordspacing

\bibitem{Jia2021Scaling}
\BIBentryALTinterwordspacing
C.~Jia, Y.~Yang, Y.~Xia, Y.~Chen, Z.~Parekh, H.~Pham, Q.~V. Le, Y.~Sung, Z.~Li,
  and T.~Duerig, ``Scaling up visual and vision-language representation
  learning with noisy text supervision,'' in \emph{Proceedings of the 38th
  International Conference on Machine Learning, {ICML} 2021, 18-24 July 2021,
  Virtual Event}, ser. Proceedings of Machine Learning Research, M.~Meila and
  T.~Zhang, Eds., vol. 139.\hskip 1em plus 0.5em minus 0.4em\relax {PMLR},
  2021, pp. 4904--4916. [Online]. Available:
  \url{http://proceedings.mlr.press/v139/jia21b.html}
\BIBentrySTDinterwordspacing

\bibitem{Huo2021WenLan}
\BIBentryALTinterwordspacing
Y.~Huo, M.~Zhang, G.~Liu, H.~Lu, Y.~Gao, G.~Yang, J.~Wen, H.~Zhang, B.~Xu,
  W.~Zheng, Z.~Xi, Y.~Yang, A.~Hu, J.~Zhao, R.~Li, Y.~Zhao, L.~Zhang, Y.~Song,
  X.~Hong, W.~Cui, D.~Y. Hou, Y.~Li, J.~Li, P.~Liu, Z.~Gong, C.~Jin, Y.~Sun,
  S.~Chen, Z.~Lu, Z.~Dou, Q.~Jin, Y.~Lan, W.~X. Zhao, R.~Song, and J.~Wen,
  ``Wenlan: Bridging vision and language by large-scale multi-modal
  pre-training,'' \emph{CoRR}, vol. abs/2103.06561, 2021. [Online]. Available:
  \url{https://arxiv.org/abs/2103.06561}
\BIBentrySTDinterwordspacing

\bibitem{Wang2021EfficientCLIP}
\BIBentryALTinterwordspacing
J.~Wang, H.~Wang, J.~Deng, W.~Wu, and D.~Zhang, ``Efficientclip: Efficient
  cross-modal pre-training by ensemble confident learning and language
  modeling,'' \emph{CoRR}, vol. abs/2109.04699, 2021. [Online]. Available:
  \url{https://arxiv.org/abs/2109.04699}
\BIBentrySTDinterwordspacing

\bibitem{Tejankar2021Fistful}
\BIBentryALTinterwordspacing
A.~Tejankar, M.~Sanjabi, B.~Wu, S.~Xie, M.~Khabsa, H.~Pirsiavash, and
  H.~Firooz, ``A fistful of words: Learning transferable visual models from
  bag-of-words supervision,'' \emph{CoRR}, vol. abs/2112.13884, 2021. [Online].
  Available: \url{https://arxiv.org/abs/2112.13884}
\BIBentrySTDinterwordspacing

\bibitem{Mu2022SLIP}
\BIBentryALTinterwordspacing
N.~Mu, A.~Kirillov, D.~A. Wagner, and S.~Xie, ``{SLIP:} self-supervision meets
  language-image pre-training,'' in \emph{Computer Vision - {ECCV} 2022 - 17th
  European Conference, Tel Aviv, Israel, October 23-27, 2022, Proceedings, Part
  {XXVI}}, ser. Lecture Notes in Computer Science, S.~Avidan, G.~J. Brostow,
  M.~Ciss{\'{e}}, G.~M. Farinella, and T.~Hassner, Eds., vol. 13686.\hskip 1em
  plus 0.5em minus 0.4em\relax Springer, 2022, pp. 529--544. [Online].
  Available: \url{https://doi.org/10.1007/978-3-031-19809-0\_30}
\BIBentrySTDinterwordspacing

\bibitem{Li2022BLIP}
\BIBentryALTinterwordspacing
J.~Li, D.~Li, C.~Xiong, and S.~C.~H. Hoi, ``{BLIP:} bootstrapping
  language-image pre-training for unified vision-language understanding and
  generation,'' in \emph{International Conference on Machine Learning, {ICML}
  2022, 17-23 July 2022, Baltimore, Maryland, {USA}}, ser. Proceedings of
  Machine Learning Research, K.~Chaudhuri, S.~Jegelka, L.~Song,
  C.~Szepesv{\'{a}}ri, G.~Niu, and S.~Sabato, Eds., vol. 162.\hskip 1em plus
  0.5em minus 0.4em\relax {PMLR}, 2022, pp. 12\,888--12\,900. [Online].
  Available: \url{https://proceedings.mlr.press/v162/li22n.html}
\BIBentrySTDinterwordspacing

\bibitem{Fuerst2022CLOOB}
\BIBentryALTinterwordspacing
A.~F{\"{u}}rst, E.~Rumetshofer, J.~Lehner, V.~T. Tran, F.~Tang, H.~Ramsauer,
  D.~P. Kreil, M.~Kopp, G.~Klambauer, A.~Bitto, and S.~Hochreiter, ``{CLOOB:}
  modern hopfield networks with infoloob outperform {CLIP},'' in
  \emph{NeurIPS}, 2022. [Online]. Available:
  \url{http://papers.nips.cc/paper\_files/paper/2022/hash/8078e76f913e31b8467e85b4c0f0d22b-Abstract-Conference.html}
\BIBentrySTDinterwordspacing

\bibitem{Chen2022Prototypical}
\BIBentryALTinterwordspacing
D.~Chen, Z.~Wu, F.~Liu, Z.~Yang, Y.~Huang, Y.~Bao, and E.~Zhou, ``Prototypical
  contrastive language image pretraining,'' \emph{CoRR}, vol. abs/2206.10996,
  2022. [Online]. Available: \url{https://doi.org/10.48550/arXiv.2206.10996}
\BIBentrySTDinterwordspacing

\bibitem{Wang2022Fengshenbang}
\BIBentryALTinterwordspacing
J.~Wang, Y.~Zhang, L.~Zhang, P.~Yang, X.~Gao, Z.~Wu, X.~Dong, J.~He, J.~Zhuo,
  Q.~Yang, Y.~Huang, X.~Li, Y.~Wu, J.~Lu, X.~Zhu, W.~Chen, T.~Han, K.~Pan,
  R.~Wang, H.~Wang, X.~Wu, Z.~Zeng, C.~Chen, R.~Gan, and J.~Zhang,
  ``Fengshenbang 1.0: Being the foundation of chinese cognitive intelligence,''
  \emph{CoRR}, vol. abs/2209.02970, 2022. [Online]. Available:
  \url{https://doi.org/10.48550/arXiv.2209.02970}
\BIBentrySTDinterwordspacing

\bibitem{Yang2022Chinese}
\BIBentryALTinterwordspacing
A.~Yang, J.~Pan, J.~Lin, R.~Men, Y.~Zhang, J.~Zhou, and C.~Zhou, ``Chinese
  {CLIP:} contrastive vision-language pretraining in chinese,'' \emph{CoRR},
  vol. abs/2211.01335, 2022. [Online]. Available:
  \url{https://doi.org/10.48550/arXiv.2211.01335}
\BIBentrySTDinterwordspacing

\bibitem{Sun2023EVA}
\BIBentryALTinterwordspacing
Q.~Sun, Y.~Fang, L.~Wu, X.~Wang, and Y.~Cao, ``{EVA-CLIP:} improved training
  techniques for {CLIP} at scale,'' \emph{CoRR}, vol. abs/2303.15389, 2023.
  [Online]. Available: \url{https://doi.org/10.48550/arXiv.2303.15389}
\BIBentrySTDinterwordspacing

\bibitem{Lin2023Multimodality}
\BIBentryALTinterwordspacing
Z.~Lin, S.~Yu, Z.~Kuang, D.~Pathak, and D.~Ramanan, ``Multimodality helps
  unimodality: Cross-modal few-shot learning with multimodal models,'' in
  \emph{{IEEE/CVF} Conference on Computer Vision and Pattern Recognition,
  {CVPR} 2023, Vancouver, BC, Canada, June 17-24, 2023}.\hskip 1em plus 0.5em
  minus 0.4em\relax {IEEE}, 2023, pp. 19\,325--19\,337. [Online]. Available:
  \url{https://doi.org/10.1109/CVPR52729.2023.01852}
\BIBentrySTDinterwordspacing

\bibitem{Lester2021Power}
\BIBentryALTinterwordspacing
B.~Lester, R.~Al{-}Rfou, and N.~Constant, ``The power of scale for
  parameter-efficient prompt tuning,'' in \emph{Proceedings of the 2021
  Conference on Empirical Methods in Natural Language Processing, {EMNLP} 2021,
  Virtual Event / Punta Cana, Dominican Republic, 7-11 November, 2021},
  M.~Moens, X.~Huang, L.~Specia, and S.~W. Yih, Eds.\hskip 1em plus 0.5em minus
  0.4em\relax Association for Computational Linguistics, 2021, pp. 3045--3059.
  [Online]. Available: \url{https://doi.org/10.18653/v1/2021.emnlp-main.243}
\BIBentrySTDinterwordspacing

\bibitem{Shin2020AutoPrompt}
\BIBentryALTinterwordspacing
T.~Shin, Y.~Razeghi, R.~L.~L. IV, E.~Wallace, and S.~Singh, ``Autoprompt:
  Eliciting knowledge from language models with automatically generated
  prompts,'' in \emph{Proceedings of the 2020 Conference on Empirical Methods
  in Natural Language Processing, {EMNLP} 2020, Online, November 16-20, 2020},
  B.~Webber, T.~Cohn, Y.~He, and Y.~Liu, Eds.\hskip 1em plus 0.5em minus
  0.4em\relax Association for Computational Linguistics, 2020, pp. 4222--4235.
  [Online]. Available: \url{https://doi.org/10.18653/v1/2020.emnlp-main.346}
\BIBentrySTDinterwordspacing

\bibitem{Li2021Prefix}
\BIBentryALTinterwordspacing
X.~L. Li and P.~Liang, ``Prefix-tuning: Optimizing continuous prompts for
  generation,'' in \emph{Proceedings of the 59th Annual Meeting of the
  Association for Computational Linguistics and the 11th International Joint
  Conference on Natural Language Processing, {ACL/IJCNLP} 2021, (Volume 1: Long
  Papers), Virtual Event, August 1-6, 2021}, C.~Zong, F.~Xia, W.~Li, and
  R.~Navigli, Eds.\hskip 1em plus 0.5em minus 0.4em\relax Association for
  Computational Linguistics, 2021, pp. 4582--4597. [Online]. Available:
  \url{https://doi.org/10.18653/v1/2021.acl-long.353}
\BIBentrySTDinterwordspacing

\bibitem{Reynolds2021Prompt}
\BIBentryALTinterwordspacing
L.~Reynolds and K.~McDonell, ``Prompt programming for large language models:
  Beyond the few-shot paradigm,'' in \emph{{CHI} '21: {CHI} Conference on Human
  Factors in Computing Systems, Virtual Event / Yokohama Japan, May 8-13, 2021,
  Extended Abstracts}, Y.~Kitamura, A.~Quigley, K.~Isbister, and T.~Igarashi,
  Eds.\hskip 1em plus 0.5em minus 0.4em\relax {ACM}, 2021, pp. 314:1--314:7.
  [Online]. Available: \url{https://doi.org/10.1145/3411763.3451760}
\BIBentrySTDinterwordspacing

\bibitem{Liu2021GPT}
\BIBentryALTinterwordspacing
X.~Liu, Y.~Zheng, Z.~Du, M.~Ding, Y.~Qian, Z.~Yang, and J.~Tang, ``{GPT}
  understands, too,'' \emph{CoRR}, vol. abs/2103.10385, 2021. [Online].
  Available: \url{https://arxiv.org/abs/2103.10385}
\BIBentrySTDinterwordspacing

\bibitem{Zhou2022Conditional}
\BIBentryALTinterwordspacing
K.~Zhou, J.~Yang, C.~C. Loy, and Z.~Liu, ``Conditional prompt learning for
  vision-language models,'' in \emph{{IEEE/CVF} Conference on Computer Vision
  and Pattern Recognition, {CVPR} 2022, New Orleans, LA, USA, June 18-24,
  2022}.\hskip 1em plus 0.5em minus 0.4em\relax {IEEE}, 2022, pp.
  16\,795--16\,804. [Online]. Available:
  \url{https://doi.org/10.1109/CVPR52688.2022.01631}
\BIBentrySTDinterwordspacing

\bibitem{Huang2022Unsupervised}
\BIBentryALTinterwordspacing
T.~Huang, J.~Chu, and F.~Wei, ``Unsupervised prompt learning for
  vision-language models,'' \emph{CoRR}, vol. abs/2204.03649, 2022. [Online].
  Available: \url{https://doi.org/10.48550/arXiv.2204.03649}
\BIBentrySTDinterwordspacing

\bibitem{Zhu2022Prompt}
\BIBentryALTinterwordspacing
B.~Zhu, Y.~Niu, Y.~Han, Y.~Wu, and H.~Zhang, ``Prompt-aligned gradient for
  prompt tuning,'' \emph{CoRR}, vol. abs/2205.14865, 2022. [Online]. Available:
  \url{https://doi.org/10.48550/arXiv.2205.14865}
\BIBentrySTDinterwordspacing

\bibitem{Chen2023PLOT}
\BIBentryALTinterwordspacing
G.~Chen, W.~Yao, X.~Song, X.~Li, Y.~Rao, and K.~Zhang, ``{PLOT:} prompt
  learning with optimal transport for vision-language models,'' in \emph{The
  Eleventh International Conference on Learning Representations, {ICLR} 2023,
  Kigali, Rwanda, May 1-5, 2023}.\hskip 1em plus 0.5em minus 0.4em\relax
  OpenReview.net, 2023. [Online]. Available:
  \url{https://openreview.net/pdf?id=zqwryBoXYnh}
\BIBentrySTDinterwordspacing

\bibitem{Lu2022Prompt}
\BIBentryALTinterwordspacing
Y.~Lu, J.~Liu, Y.~Zhang, Y.~Liu, and X.~Tian, ``Prompt distribution learning,''
  in \emph{{IEEE/CVF} Conference on Computer Vision and Pattern Recognition,
  {CVPR} 2022, New Orleans, LA, USA, June 18-24, 2022}.\hskip 1em plus 0.5em
  minus 0.4em\relax {IEEE}, 2022, pp. 5196--5205. [Online]. Available:
  \url{https://doi.org/10.1109/CVPR52688.2022.00514}
\BIBentrySTDinterwordspacing

\bibitem{Ding2022Prompt}
\BIBentryALTinterwordspacing
K.~Ding, Y.~Wang, P.~Liu, Q.~Yu, H.~Zhang, S.~Xiang, and C.~Pan, ``Prompt
  tuning with soft context sharing for vision-language models,'' \emph{CoRR},
  vol. abs/2208.13474, 2022. [Online]. Available:
  \url{https://doi.org/10.48550/arXiv.2208.13474}
\BIBentrySTDinterwordspacing

\bibitem{Jia2022Visual}
\BIBentryALTinterwordspacing
M.~Jia, L.~Tang, B.~Chen, C.~Cardie, S.~J. Belongie, B.~Hariharan, and S.~Lim,
  ``Visual prompt tuning,'' in \emph{Computer Vision - {ECCV} 2022 - 17th
  European Conference, Tel Aviv, Israel, October 23-27, 2022, Proceedings, Part
  {XXXIII}}, ser. Lecture Notes in Computer Science, S.~Avidan, G.~J. Brostow,
  M.~Ciss{\'{e}}, G.~M. Farinella, and T.~Hassner, Eds., vol. 13693.\hskip 1em
  plus 0.5em minus 0.4em\relax Springer, 2022, pp. 709--727. [Online].
  Available: \url{https://doi.org/10.1007/978-3-031-19827-4\_41}
\BIBentrySTDinterwordspacing

\bibitem{Xing2022Class}
\BIBentryALTinterwordspacing
Y.~Xing, Q.~Wu, D.~Cheng, S.~Zhang, G.~Liang, and Y.~Zhang, ``Class-aware
  visual prompt tuning for vision-language pre-trained model,'' \emph{CoRR},
  vol. abs/2208.08340, 2022. [Online]. Available:
  \url{https://doi.org/10.48550/arXiv.2208.08340}
\BIBentrySTDinterwordspacing

\bibitem{Bahng2022Exploring}
H.~Bahng, A.~Jahanian, S.~Sankaranarayanan, and P.~Isola, ``Exploring visual
  prompts for adapting large-scale models,'' 2022.

\bibitem{Wu2022Unleashing}
\BIBentryALTinterwordspacing
J.~Wu, X.~Li, C.~Wei, H.~Wang, A.~L. Yuille, Y.~Zhou, and C.~Xie, ``Unleashing
  the power of visual prompting at the pixel level,'' \emph{CoRR}, vol.
  abs/2212.10556, 2022. [Online]. Available:
  \url{https://doi.org/10.48550/arXiv.2212.10556}
\BIBentrySTDinterwordspacing

\bibitem{Khattak2022MaPLe}
\BIBentryALTinterwordspacing
M.~U. Khattak, H.~A. Rasheed, M.~Maaz, S.~Khan, and F.~S. Khan, ``Maple:
  Multi-modal prompt learning,'' \emph{CoRR}, vol. abs/2210.03117, 2022.
  [Online]. Available: \url{https://doi.org/10.48550/arXiv.2210.03117}
\BIBentrySTDinterwordspacing

\bibitem{Zang2022Unified}
\BIBentryALTinterwordspacing
Y.~Zang, W.~Li, K.~Zhou, C.~Huang, and C.~C. Loy, ``Unified vision and language
  prompt learning,'' \emph{CoRR}, vol. abs/2210.07225, 2022. [Online].
  Available: \url{https://doi.org/10.48550/arXiv.2210.07225}
\BIBentrySTDinterwordspacing

\bibitem{Shen2022Multitask}
\BIBentryALTinterwordspacing
S.~Shen, S.~Yang, T.~Zhang, B.~Zhai, J.~E. Gonzalez, K.~Keutzer, and
  T.~Darrell, ``Multitask vision-language prompt tuning,'' \emph{CoRR}, vol.
  abs/2211.11720, 2022. [Online]. Available:
  \url{https://doi.org/10.48550/arXiv.2211.11720}
\BIBentrySTDinterwordspacing

\bibitem{Houlsby2019Parameter}
\BIBentryALTinterwordspacing
N.~Houlsby, A.~Giurgiu, S.~Jastrzebski, B.~Morrone, Q.~de~Laroussilhe,
  A.~Gesmundo, M.~Attariyan, and S.~Gelly, ``Parameter-efficient transfer
  learning for {NLP},'' in \emph{Proceedings of the 36th International
  Conference on Machine Learning, {ICML} 2019, 9-15 June 2019, Long Beach,
  California, {USA}}, ser. Proceedings of Machine Learning Research,
  K.~Chaudhuri and R.~Salakhutdinov, Eds., vol.~97.\hskip 1em plus 0.5em minus
  0.4em\relax {PMLR}, 2019, pp. 2790--2799. [Online]. Available:
  \url{http://proceedings.mlr.press/v97/houlsby19a.html}
\BIBentrySTDinterwordspacing

\bibitem{Ding2022Delta}
\BIBentryALTinterwordspacing
N.~Ding, Y.~Qin, G.~Yang, F.~Wei, Z.~Yang, Y.~Su, S.~Hu, Y.~Chen, C.~Chan,
  W.~Chen, J.~Yi, W.~Zhao, X.~Wang, Z.~Liu, H.~Zheng, J.~Chen, Y.~Liu, J.~Tang,
  J.~Li, and M.~Sun, ``Delta tuning: {A} comprehensive study of parameter
  efficient methods for pre-trained language models,'' \emph{CoRR}, vol.
  abs/2203.06904, 2022. [Online]. Available:
  \url{https://doi.org/10.48550/arXiv.2203.06904}
\BIBentrySTDinterwordspacing

\bibitem{Jie2022Convolutional}
\BIBentryALTinterwordspacing
S.~Jie and Z.~Deng, ``Convolutional bypasses are better vision transformer
  adapters,'' \emph{CoRR}, vol. abs/2207.07039, 2022. [Online]. Available:
  \url{https://doi.org/10.48550/arXiv.2207.07039}
\BIBentrySTDinterwordspacing

\bibitem{Jie2023FacT}
\BIBentryALTinterwordspacing
------, ``Fact: Factor-tuning for lightweight adaptation on vision
  transformer,'' in \emph{Thirty-Seventh {AAAI} Conference on Artificial
  Intelligence, {AAAI} 2023, Thirty-Fifth Conference on Innovative Applications
  of Artificial Intelligence, {IAAI} 2023, Thirteenth Symposium on Educational
  Advances in Artificial Intelligence, {EAAI} 2023, Washington, DC, USA,
  February 7-14, 2023}, B.~Williams, Y.~Chen, and J.~Neville, Eds.\hskip 1em
  plus 0.5em minus 0.4em\relax {AAAI} Press, 2023, pp. 1060--1068. [Online].
  Available: \url{https://ojs.aaai.org/index.php/AAAI/article/view/25187}
\BIBentrySTDinterwordspacing

\bibitem{Zhang2021Tip}
\BIBentryALTinterwordspacing
R.~Zhang, R.~Fang, W.~Zhang, P.~Gao, K.~Li, J.~Dai, Y.~Qiao, and H.~Li,
  ``Tip-adapter: Training-free clip-adapter for better vision-language
  modeling,'' \emph{CoRR}, vol. abs/2111.03930, 2021. [Online]. Available:
  \url{https://arxiv.org/abs/2111.03930}
\BIBentrySTDinterwordspacing

\bibitem{Pantazis2022SVL}
\BIBentryALTinterwordspacing
O.~Pantazis, G.~J. Brostow, K.~E. Jones, and O.~M. Aodha, ``Svl-adapter:
  Self-supervised adapter for vision-language pretrained models,'' in
  \emph{33rd British Machine Vision Conference 2022, {BMVC} 2022, London, UK,
  November 21-24, 2022}.\hskip 1em plus 0.5em minus 0.4em\relax {BMVA} Press,
  2022, p. 580. [Online]. Available: \url{https://bmvc2022.mpi-inf.mpg.de/580/}
\BIBentrySTDinterwordspacing

\bibitem{Peng2022SgVA}
\BIBentryALTinterwordspacing
F.~Peng, X.~Yang, and C.~Xu, ``Sgva-clip: Semantic-guided visual adapting of
  vision-language models for few-shot image classification,'' \emph{CoRR}, vol.
  abs/2211.16191, 2022. [Online]. Available:
  \url{https://doi.org/10.48550/arXiv.2211.16191}
\BIBentrySTDinterwordspacing

\bibitem{Guo2023CALIP}
\BIBentryALTinterwordspacing
Z.~Guo, R.~Zhang, L.~Qiu, X.~Ma, X.~Miao, X.~He, and B.~Cui, ``{CALIP:}
  zero-shot enhancement of {CLIP} with parameter-free attention,'' in
  \emph{Thirty-Seventh {AAAI} Conference on Artificial Intelligence, {AAAI}
  2023, Thirty-Fifth Conference on Innovative Applications of Artificial
  Intelligence, {IAAI} 2023, Thirteenth Symposium on Educational Advances in
  Artificial Intelligence, {EAAI} 2023, Washington, DC, USA, February 7-14,
  2023}, B.~Williams, Y.~Chen, and J.~Neville, Eds.\hskip 1em plus 0.5em minus
  0.4em\relax {AAAI} Press, 2023, pp. 746--754. [Online]. Available:
  \url{https://ojs.aaai.org/index.php/AAAI/article/view/25152}
\BIBentrySTDinterwordspacing

\bibitem{Jiang2022Cross}
\BIBentryALTinterwordspacing
H.~Jiang, J.~Zhang, R.~Huang, C.~Ge, Z.~Ni, J.~Lu, J.~Zhou, S.~Song, and
  G.~Huang, ``Cross-modal adapter for text-video retrieval,'' \emph{CoRR}, vol.
  abs/2211.09623, 2022. [Online]. Available:
  \url{https://doi.org/10.48550/arXiv.2211.09623}
\BIBentrySTDinterwordspacing

\bibitem{Zhang2023Multimodal}
\BIBentryALTinterwordspacing
B.~Zhang, X.~Jin, W.~Gong, K.~Xu, Z.~Zhang, P.~Wang, X.~Shen, and J.~Feng,
  ``Multimodal video adapter for parameter efficient video text retrieval,''
  \emph{CoRR}, vol. abs/2301.07868, 2023. [Online]. Available:
  \url{https://doi.org/10.48550/arXiv.2301.07868}
\BIBentrySTDinterwordspacing

\bibitem{Lu2023UniAdapter}
\BIBentryALTinterwordspacing
H.~Lu, M.~Ding, Y.~Huo, G.~Yang, Z.~Lu, M.~Tomizuka, and W.~Zhan, ``Uniadapter:
  Unified parameter-efficient transfer learning for cross-modal modeling,''
  \emph{CoRR}, vol. abs/2302.06605, 2023. [Online]. Available:
  \url{https://doi.org/10.48550/arXiv.2302.06605}
\BIBentrySTDinterwordspacing

\bibitem{Zhang2019ERNIE}
\BIBentryALTinterwordspacing
Z.~Zhang, X.~Han, Z.~Liu, X.~Jiang, M.~Sun, and Q.~Liu, ``{ERNIE:} enhanced
  language representation with informative entities,'' in \emph{Proceedings of
  the 57th Conference of the Association for Computational Linguistics, {ACL}
  2019, Florence, Italy, July 28- August 2, 2019, Volume 1: Long Papers},
  A.~Korhonen, D.~R. Traum, and L.~M{\`{a}}rquez, Eds.\hskip 1em plus 0.5em
  minus 0.4em\relax Association for Computational Linguistics, 2019, pp.
  1441--1451. [Online]. Available: \url{https://doi.org/10.18653/v1/p19-1139}
\BIBentrySTDinterwordspacing

\bibitem{Liu2020K}
\BIBentryALTinterwordspacing
W.~Liu, P.~Zhou, Z.~Zhao, Z.~Wang, Q.~Ju, H.~Deng, and P.~Wang, ``{K-BERT:}
  enabling language representation with knowledge graph,'' in \emph{The
  Thirty-Fourth {AAAI} Conference on Artificial Intelligence, {AAAI} 2020, The
  Thirty-Second Innovative Applications of Artificial Intelligence Conference,
  {IAAI} 2020, The Tenth {AAAI} Symposium on Educational Advances in Artificial
  Intelligence, {EAAI} 2020, New York, NY, USA, February 7-12, 2020}.\hskip 1em
  plus 0.5em minus 0.4em\relax {AAAI} Press, 2020, pp. 2901--2908. [Online].
  Available: \url{https://ojs.aaai.org/index.php/AAAI/article/view/5681}
\BIBentrySTDinterwordspacing

\bibitem{Yu2021ERNIE}
\BIBentryALTinterwordspacing
F.~Yu, J.~Tang, W.~Yin, Y.~Sun, H.~Tian, H.~Wu, and H.~Wang, ``Ernie-vil:
  Knowledge enhanced vision-language representations through scene graphs,'' in
  \emph{Thirty-Fifth {AAAI} Conference on Artificial Intelligence, {AAAI} 2021,
  Thirty-Third Conference on Innovative Applications of Artificial
  Intelligence, {IAAI} 2021, The Eleventh Symposium on Educational Advances in
  Artificial Intelligence, {EAAI} 2021, Virtual Event, February 2-9,
  2021}.\hskip 1em plus 0.5em minus 0.4em\relax {AAAI} Press, 2021, pp.
  3208--3216. [Online]. Available:
  \url{https://ojs.aaai.org/index.php/AAAI/article/view/16431}
\BIBentrySTDinterwordspacing

\bibitem{Xing2021KM}
\BIBentryALTinterwordspacing
Y.~Xing, Z.~Shi, Z.~Meng, G.~Lakemeyer, Y.~Ma, and R.~Wattenhofer, ``{KM-BART:}
  knowledge enhanced multimodal {BART} for visual commonsense generation,'' in
  \emph{Proceedings of the 59th Annual Meeting of the Association for
  Computational Linguistics and the 11th International Joint Conference on
  Natural Language Processing, {ACL/IJCNLP} 2021, (Volume 1: Long Papers),
  Virtual Event, August 1-6, 2021}, C.~Zong, F.~Xia, W.~Li, and R.~Navigli,
  Eds.\hskip 1em plus 0.5em minus 0.4em\relax Association for Computational
  Linguistics, 2021, pp. 525--535. [Online]. Available:
  \url{https://doi.org/10.18653/v1/2021.acl-long.44}
\BIBentrySTDinterwordspacing

\bibitem{Shen2022K}
\BIBentryALTinterwordspacing
S.~Shen, C.~Li, X.~Hu, Y.~Xie, J.~Yang, P.~Zhang, Z.~Gan, L.~Wang, L.~Yuan,
  C.~Liu, K.~Keutzer, T.~Darrell, A.~Rohrbach, and J.~Gao, ``{K-LITE:} learning
  transferable visual models with external knowledge,'' in \emph{NeurIPS},
  2022. [Online]. Available:
  \url{http://papers.nips.cc/paper\_files/paper/2022/hash/63fef0802863f47775c3563e18cbba17-Abstract-Conference.html}
\BIBentrySTDinterwordspacing

\bibitem{Menon2023Visual}
\BIBentryALTinterwordspacing
S.~Menon and C.~Vondrick, ``Visual classification via description from large
  language models,'' in \emph{The Eleventh International Conference on Learning
  Representations, {ICLR} 2023, Kigali, Rwanda, May 1-5, 2023}.\hskip 1em plus
  0.5em minus 0.4em\relax OpenReview.net, 2023. [Online]. Available:
  \url{https://openreview.net/pdf?id=jlAjNL8z5cs}
\BIBentrySTDinterwordspacing

\bibitem{Yang2022Language}
\BIBentryALTinterwordspacing
Y.~Yang, A.~Panagopoulou, S.~Zhou, D.~Jin, C.~Callison{-}Burch, and M.~Yatskar,
  ``Language in a bottle: Language model guided concept bottlenecks for
  interpretable image classification,'' \emph{CoRR}, vol. abs/2211.11158, 2022.
  [Online]. Available: \url{https://doi.org/10.48550/arXiv.2211.11158}
\BIBentrySTDinterwordspacing

\bibitem{Li2022ELEVATER}
\BIBentryALTinterwordspacing
C.~Li, H.~Liu, L.~H. Li, P.~Zhang, J.~Aneja, J.~Yang, P.~Jin, H.~Hu, Z.~Liu,
  Y.~J. Lee, and J.~Gao, ``{ELEVATER:} {A} benchmark and toolkit for evaluating
  language-augmented visual models,'' in \emph{NeurIPS}, 2022. [Online].
  Available:
  \url{http://papers.nips.cc/paper\_files/paper/2022/hash/3c4688b6a76f25f2311daa0d75a58f1a-Abstract-Datasets\_and\_Benchmarks.html}
\BIBentrySTDinterwordspacing

\bibitem{Udandarao2022SuS}
\BIBentryALTinterwordspacing
V.~Udandarao, A.~Gupta, and S.~Albanie, ``Sus-x: Training-free name-only
  transfer of vision-language models,'' \emph{CoRR}, vol. abs/2211.16198, 2022.
  [Online]. Available: \url{https://doi.org/10.48550/arXiv.2211.16198}
\BIBentrySTDinterwordspacing

\bibitem{Schuhmann2022LAION}
\BIBentryALTinterwordspacing
C.~Schuhmann, R.~Beaumont, R.~Vencu, C.~Gordon, R.~Wightman, M.~Cherti,
  T.~Coombes, A.~Katta, C.~Mullis, M.~Wortsman, P.~Schramowski, S.~Kundurthy,
  K.~Crowson, L.~Schmidt, R.~Kaczmarczyk, and J.~Jitsev, ``{LAION-5B:} an open
  large-scale dataset for training next generation image-text models,'' in
  \emph{NeurIPS}, 2022. [Online]. Available:
  \url{http://papers.nips.cc/paper\_files/paper/2022/hash/a1859debfb3b59d094f3504d5ebb6c25-Abstract-Datasets\_and\_Benchmarks.html}
\BIBentrySTDinterwordspacing

\bibitem{Rombach2022High}
\BIBentryALTinterwordspacing
R.~Rombach, A.~Blattmann, D.~Lorenz, P.~Esser, and B.~Ommer, ``High-resolution
  image synthesis with latent diffusion models,'' in \emph{{IEEE/CVF}
  Conference on Computer Vision and Pattern Recognition, {CVPR} 2022, New
  Orleans, LA, USA, June 18-24, 2022}.\hskip 1em plus 0.5em minus 0.4em\relax
  {IEEE}, 2022, pp. 10\,674--10\,685. [Online]. Available:
  \url{https://doi.org/10.1109/CVPR52688.2022.01042}
\BIBentrySTDinterwordspacing

\bibitem{Wortsman2022Robust}
\BIBentryALTinterwordspacing
M.~Wortsman, G.~Ilharco, J.~W. Kim, M.~Li, S.~Kornblith, R.~Roelofs, R.~G.
  Lopes, H.~Hajishirzi, A.~Farhadi, H.~Namkoong, and L.~Schmidt, ``Robust
  fine-tuning of zero-shot models,'' in \emph{{IEEE/CVF} Conference on Computer
  Vision and Pattern Recognition, {CVPR} 2022, New Orleans, LA, USA, June
  18-24, 2022}.\hskip 1em plus 0.5em minus 0.4em\relax {IEEE}, 2022, pp.
  7949--7961. [Online]. Available:
  \url{https://doi.org/10.1109/CVPR52688.2022.00780}
\BIBentrySTDinterwordspacing

\bibitem{Zhou2022Extract}
\BIBentryALTinterwordspacing
C.~Zhou, C.~C. Loy, and B.~Dai, ``Extract free dense labels from {CLIP},'' in
  \emph{Computer Vision - {ECCV} 2022 - 17th European Conference, Tel Aviv,
  Israel, October 23-27, 2022, Proceedings, Part {XXVIII}}, ser. Lecture Notes
  in Computer Science, S.~Avidan, G.~J. Brostow, M.~Ciss{\'{e}}, G.~M.
  Farinella, and T.~Hassner, Eds., vol. 13688.\hskip 1em plus 0.5em minus
  0.4em\relax Springer, 2022, pp. 696--712. [Online]. Available:
  \url{https://doi.org/10.1007/978-3-031-19815-1\_40}
\BIBentrySTDinterwordspacing

\bibitem{Zhang2021VT}
\BIBentryALTinterwordspacing
R.~Zhang, L.~Qiu, W.~Zhang, and Z.~Zeng, ``{VT-CLIP:} enhancing vision-language
  models with visual-guided texts,'' \emph{CoRR}, vol. abs/2112.02399, 2021.
  [Online]. Available: \url{https://arxiv.org/abs/2112.02399}
\BIBentrySTDinterwordspacing

\bibitem{Deng2009ImageNet}
\BIBentryALTinterwordspacing
J.~Deng, W.~Dong, R.~Socher, L.~Li, K.~Li, and L.~Fei{-}Fei, ``Imagenet: {A}
  large-scale hierarchical image database,'' in \emph{2009 {IEEE} Computer
  Society Conference on Computer Vision and Pattern Recognition {(CVPR} 2009),
  20-25 June 2009, Miami, Florida, {USA}}.\hskip 1em plus 0.5em minus
  0.4em\relax {IEEE} Computer Society, 2009, pp. 248--255. [Online]. Available:
  \url{https://doi.org/10.1109/CVPR.2009.5206848}
\BIBentrySTDinterwordspacing

\bibitem{FeiFei2007Learning}
\BIBentryALTinterwordspacing
L.~Fei{-}Fei, R.~Fergus, and P.~Perona, ``Learning generative visual models
  from few training examples: An incremental bayesian approach tested on 101
  object categories,'' \emph{Comput. Vis. Image Underst.}, vol. 106, no.~1, pp.
  59--70, 2007. [Online]. Available:
  \url{https://doi.org/10.1016/j.cviu.2005.09.012}
\BIBentrySTDinterwordspacing

\bibitem{Parkhi2012Cats}
\BIBentryALTinterwordspacing
O.~M. Parkhi, A.~Vedaldi, A.~Zisserman, and C.~V. Jawahar, ``Cats and dogs,''
  in \emph{2012 {IEEE} Conference on Computer Vision and Pattern Recognition,
  Providence, RI, USA, June 16-21, 2012}.\hskip 1em plus 0.5em minus
  0.4em\relax {IEEE} Computer Society, 2012, pp. 3498--3505. [Online].
  Available: \url{https://doi.org/10.1109/CVPR.2012.6248092}
\BIBentrySTDinterwordspacing

\bibitem{Krause20133D}
\BIBentryALTinterwordspacing
J.~Krause, M.~Stark, J.~Deng, and L.~Fei{-}Fei, ``3d object representations for
  fine-grained categorization,'' in \emph{2013 {IEEE} International Conference
  on Computer Vision Workshops, {ICCV} Workshops 2013, Sydney, Australia,
  December 1-8, 2013}.\hskip 1em plus 0.5em minus 0.4em\relax {IEEE} Computer
  Society, 2013, pp. 554--561. [Online]. Available:
  \url{https://doi.org/10.1109/ICCVW.2013.77}
\BIBentrySTDinterwordspacing

\bibitem{Nilsback2008Automated}
\BIBentryALTinterwordspacing
M.~Nilsback and A.~Zisserman, ``Automated flower classification over a large
  number of classes,'' in \emph{Sixth Indian Conference on Computer Vision,
  Graphics {\&} Image Processing, {ICVGIP} 2008, Bhubaneswar, India, 16-19
  December 2008}.\hskip 1em plus 0.5em minus 0.4em\relax {IEEE} Computer
  Society, 2008, pp. 722--729. [Online]. Available:
  \url{https://doi.org/10.1109/ICVGIP.2008.47}
\BIBentrySTDinterwordspacing

\bibitem{Bossard2014Food}
\BIBentryALTinterwordspacing
L.~Bossard, M.~Guillaumin, and L.~V. Gool, ``Food-101 - mining discriminative
  components with random forests,'' in \emph{Computer Vision - {ECCV} 2014 -
  13th European Conference, Zurich, Switzerland, September 6-12, 2014,
  Proceedings, Part {VI}}, ser. Lecture Notes in Computer Science, D.~J. Fleet,
  T.~Pajdla, B.~Schiele, and T.~Tuytelaars, Eds., vol. 8694.\hskip 1em plus
  0.5em minus 0.4em\relax Springer, 2014, pp. 446--461. [Online]. Available:
  \url{https://doi.org/10.1007/978-3-319-10599-4\_29}
\BIBentrySTDinterwordspacing

\bibitem{Maji2013Fine}
\BIBentryALTinterwordspacing
S.~Maji, E.~Rahtu, J.~Kannala, M.~B. Blaschko, and A.~Vedaldi, ``Fine-grained
  visual classification of aircraft,'' \emph{CoRR}, vol. abs/1306.5151, 2013.
  [Online]. Available: \url{http://arxiv.org/abs/1306.5151}
\BIBentrySTDinterwordspacing

\bibitem{Xiao2010SUN}
\BIBentryALTinterwordspacing
J.~Xiao, J.~Hays, K.~A. Ehinger, A.~Oliva, and A.~Torralba, ``{SUN} database:
  Large-scale scene recognition from abbey to zoo,'' in \emph{The Twenty-Third
  {IEEE} Conference on Computer Vision and Pattern Recognition, {CVPR} 2010,
  San Francisco, CA, USA, 13-18 June 2010}.\hskip 1em plus 0.5em minus
  0.4em\relax {IEEE} Computer Society, 2010, pp. 3485--3492. [Online].
  Available: \url{https://doi.org/10.1109/CVPR.2010.5539970}
\BIBentrySTDinterwordspacing

\bibitem{Soomro2012UCF101}
\BIBentryALTinterwordspacing
K.~Soomro, A.~R. Zamir, and M.~Shah, ``{UCF101:} {A} dataset of 101 human
  actions classes from videos in the wild,'' \emph{CoRR}, vol. abs/1212.0402,
  2012. [Online]. Available: \url{http://arxiv.org/abs/1212.0402}
\BIBentrySTDinterwordspacing

\bibitem{Cimpoi2014Describing}
\BIBentryALTinterwordspacing
M.~Cimpoi, S.~Maji, I.~Kokkinos, S.~Mohamed, and A.~Vedaldi, ``Describing
  textures in the wild,'' in \emph{2014 {IEEE} Conference on Computer Vision
  and Pattern Recognition, {CVPR} 2014, Columbus, OH, USA, June 23-28,
  2014}.\hskip 1em plus 0.5em minus 0.4em\relax {IEEE} Computer Society, 2014,
  pp. 3606--3613. [Online]. Available:
  \url{https://doi.org/10.1109/CVPR.2014.461}
\BIBentrySTDinterwordspacing

\bibitem{Helber2019EuroSAT}
\BIBentryALTinterwordspacing
P.~Helber, B.~Bischke, A.~Dengel, and D.~Borth, ``Eurosat: {A} novel dataset
  and deep learning benchmark for land use and land cover classification,''
  \emph{{IEEE} J. Sel. Top. Appl. Earth Obs. Remote. Sens.}, vol.~12, no.~7,
  pp. 2217--2226, 2019. [Online]. Available:
  \url{https://doi.org/10.1109/JSTARS.2019.2918242}
\BIBentrySTDinterwordspacing

\bibitem{Recht2019Do}
\BIBentryALTinterwordspacing
B.~Recht, R.~Roelofs, L.~Schmidt, and V.~Shankar, ``Do imagenet classifiers
  generalize to imagenet?'' in \emph{Proceedings of the 36th International
  Conference on Machine Learning, {ICML} 2019, 9-15 June 2019, Long Beach,
  California, {USA}}, ser. Proceedings of Machine Learning Research,
  K.~Chaudhuri and R.~Salakhutdinov, Eds., vol.~97.\hskip 1em plus 0.5em minus
  0.4em\relax {PMLR}, 2019, pp. 5389--5400. [Online]. Available:
  \url{http://proceedings.mlr.press/v97/recht19a.html}
\BIBentrySTDinterwordspacing

\bibitem{Wang2019Learning}
\BIBentryALTinterwordspacing
H.~Wang, S.~Ge, Z.~C. Lipton, and E.~P. Xing, ``Learning robust global
  representations by penalizing local predictive power,'' in \emph{Advances in
  Neural Information Processing Systems 32: Annual Conference on Neural
  Information Processing Systems 2019, NeurIPS 2019, December 8-14, 2019,
  Vancouver, BC, Canada}, H.~M. Wallach, H.~Larochelle, A.~Beygelzimer,
  F.~d'Alch{\'{e}}{-}Buc, E.~B. Fox, and R.~Garnett, Eds., 2019, pp.
  10\,506--10\,518. [Online]. Available:
  \url{https://proceedings.neurips.cc/paper/2019/hash/3eefceb8087e964f89c2d59e8a249915-Abstract.html}
\BIBentrySTDinterwordspacing

\bibitem{Hendrycks2021Natural}
D.~Hendrycks, K.~Zhao, S.~Basart, J.~Steinhardt, and D.~Song, ``Natural
  adversarial examples,'' in \emph{2021 IEEE/CVF Conference on Computer Vision
  and Pattern Recognition (CVPR)}, 2021, pp. 15\,257--15\,266.

\bibitem{Hendrycks2021Many}
\BIBentryALTinterwordspacing
D.~Hendrycks, S.~Basart, N.~Mu, S.~Kadavath, F.~Wang, E.~Dorundo, R.~Desai,
  T.~Zhu, S.~Parajuli, M.~Guo, D.~Song, J.~Steinhardt, and J.~Gilmer, ``The
  many faces of robustness: {A} critical analysis of out-of-distribution
  generalization,'' in \emph{2021 {IEEE/CVF} International Conference on
  Computer Vision, {ICCV} 2021, Montreal, QC, Canada, October 10-17,
  2021}.\hskip 1em plus 0.5em minus 0.4em\relax {IEEE}, 2021, pp. 8320--8329.
  [Online]. Available: \url{https://doi.org/10.1109/ICCV48922.2021.00823}
\BIBentrySTDinterwordspacing

\bibitem{Zhang2023Vision}
\BIBentryALTinterwordspacing
J.~Zhang, J.~Huang, S.~Jin, and S.~Lu, ``Vision-language models for vision
  tasks: {A} survey,'' \emph{CoRR}, vol. abs/2304.00685, 2023. [Online].
  Available: \url{https://doi.org/10.48550/arXiv.2304.00685}
\BIBentrySTDinterwordspacing

\bibitem{Shu2022Test}
\BIBentryALTinterwordspacing
M.~Shu, W.~Nie, D.~Huang, Z.~Yu, T.~Goldstein, A.~Anandkumar, and C.~Xiao,
  ``Test-time prompt tuning for zero-shot generalization in vision-language
  models,'' in \emph{NeurIPS}, 2022. [Online]. Available:
  \url{http://papers.nips.cc/paper\_files/paper/2022/hash/5bf2b802e24106064dc547ae9283bb0c-Abstract-Conference.html}
\BIBentrySTDinterwordspacing

\bibitem{Ma2023Understanding}
\BIBentryALTinterwordspacing
C.~Ma, Y.~Liu, J.~Deng, L.~Xie, W.~Dong, and C.~Xu, ``Understanding and
  mitigating overfitting in prompt tuning for vision-language models,''
  \emph{{IEEE} Trans. Circuits Syst. Video Technol.}, vol.~33, no.~9, pp.
  4616--4629, 2023. [Online]. Available:
  \url{https://doi.org/10.1109/TCSVT.2023.3245584}
\BIBentrySTDinterwordspacing

\bibitem{Li2023Gradient}
\BIBentryALTinterwordspacing
J.~Li, M.~Gao, L.~Wei, S.~Tang, W.~Zhang, M.~Li, W.~Ji, Q.~Tian, T.~Chua, and
  Y.~Zhuang, ``Gradient-regulated meta-prompt learning for generalizable
  vision-language models,'' \emph{CoRR}, vol. abs/2303.06571, 2023. [Online].
  Available: \url{https://doi.org/10.48550/arXiv.2303.06571}
\BIBentrySTDinterwordspacing

\bibitem{Yao2023Visual}
\BIBentryALTinterwordspacing
H.~Yao, R.~Zhang, and C.~Xu, ``Visual-language prompt tuning with
  knowledge-guided context optimization,'' in \emph{{IEEE/CVF} Conference on
  Computer Vision and Pattern Recognition, {CVPR} 2023, Vancouver, BC, Canada,
  June 17-24, 2023}.\hskip 1em plus 0.5em minus 0.4em\relax {IEEE}, 2023, pp.
  6757--6767. [Online]. Available:
  \url{https://doi.org/10.1109/CVPR52729.2023.00653}
\BIBentrySTDinterwordspacing

\bibitem{Liu2023Patch}
\BIBentryALTinterwordspacing
X.~Liu, D.~Wang, M.~Li, Z.~Duan, Y.~Xu, B.~Chen, and M.~Zhou, ``Patch-token
  aligned bayesian prompt learning for vision-language models,'' \emph{CoRR},
  vol. abs/2303.09100, 2023. [Online]. Available:
  \url{https://doi.org/10.48550/arXiv.2303.09100}
\BIBentrySTDinterwordspacing

\bibitem{Bulat2023LASP}
\BIBentryALTinterwordspacing
A.~Bulat and G.~Tzimiropoulos, ``{LASP:} text-to-text optimization for
  language-aware soft prompting of vision {\&} language models,'' in
  \emph{{IEEE/CVF} Conference on Computer Vision and Pattern Recognition,
  {CVPR} 2023, Vancouver, BC, Canada, June 17-24, 2023}.\hskip 1em plus 0.5em
  minus 0.4em\relax {IEEE}, 2023, pp. 23\,232--23\,241. [Online]. Available:
  \url{https://doi.org/10.1109/CVPR52729.2023.02225}
\BIBentrySTDinterwordspacing

\bibitem{Maurer2004Note}
\BIBentryALTinterwordspacing
A.~Maurer, ``A note on the {PAC} bayesian theorem,'' \emph{CoRR}, vol.
  cs.LG/0411099, 2004. [Online]. Available:
  \url{http://arxiv.org/abs/cs.LG/0411099}
\BIBentrySTDinterwordspacing

\bibitem{Sicilia2022PAC}
\BIBentryALTinterwordspacing
A.~Sicilia, K.~Atwell, M.~Alikhani, and S.~J. Hwang, ``Pac-bayesian domain
  adaptation bounds for multiclass learners,'' in \emph{Uncertainty in
  Artificial Intelligence, Proceedings of the Thirty-Eighth Conference on
  Uncertainty in Artificial Intelligence, {UAI} 2022, 1-5 August 2022,
  Eindhoven, The Netherlands}, ser. Proceedings of Machine Learning Research,
  J.~Cussens and K.~Zhang, Eds., vol. 180.\hskip 1em plus 0.5em minus
  0.4em\relax {PMLR}, 2022, pp. 1824--1834. [Online]. Available:
  \url{https://proceedings.mlr.press/v180/sicilia22a.html}
\BIBentrySTDinterwordspacing

\bibitem{Crammer2006Learning}
\BIBentryALTinterwordspacing
K.~Crammer, M.~J. Kearns, and J.~Wortman, ``Learning from multiple sources,''
  in \emph{Advances in Neural Information Processing Systems 19, Proceedings of
  the Twentieth Annual Conference on Neural Information Processing Systems,
  Vancouver, British Columbia, Canada, December 4-7, 2006}, B.~Sch{\"{o}}lkopf,
  J.~C. Platt, and T.~Hofmann, Eds.\hskip 1em plus 0.5em minus 0.4em\relax
  {MIT} Press, 2006, pp. 321--328. [Online]. Available:
  \url{https://proceedings.neurips.cc/paper/2006/hash/0f21f0349462cacdc5796990d37760ae-Abstract.html}
\BIBentrySTDinterwordspacing

\bibitem{Kang2023Neural}
\BIBentryALTinterwordspacing
J.~Kang, J.~K. Kang, J.~Kim, K.~Jeon, H.~Chung, and B.~Park, ``Neural
  architecture search survey: {A} computer vision perspective,''
  \emph{Sensors}, vol.~23, no.~3, p. 1713, 2023. [Online]. Available:
  \url{https://doi.org/10.3390/s23031713}
\BIBentrySTDinterwordspacing

\end{thebibliography}

\end{document}